%% file: main.tex
\documentclass[11pt]{article}

\usepackage{fullpage,times,url,bm}

\usepackage{amsthm,amsfonts,amsmath,amssymb,epsfig,color,float,graphicx,verbatim}
\usepackage{algorithm}
\usepackage{algpseudocode} 
\algblock[Name]{Stage}{EndStage}
\algrenewtext{Stage}{\underline{\textbf{Stage 1 (SubGM)}}:}
\algnotext[Name]{EndStage}
\algblock[Name]{Stagetwo}{EndStagetwo}
\algrenewtext{Stagetwo}{\underline{\textbf{Stage 2} (optional)}:}
\algnotext[Name]{EndStagetwo}
\usepackage{makecell}
\usepackage{hhline}
\usepackage{colortbl}
\usepackage{bbm}
\usepackage{enumitem}

\usepackage{graphicx}
\usepackage{subcaption}

\usepackage{array}

\usepackage[dvipsnames]{xcolor}

\usepackage[colorlinks,linkcolor=purple, citecolor=Green, urlcolor=magenta, anchorcolor=ForestGreen,bookmarks=False]{hyperref}
\usepackage{cleveref}

\newtheorem{theorem}{Theorem}[section]
\newtheorem{proposition}{Proposition}[section]
\newtheorem{lemma}{Lemma}[section]

\newtheorem{definition}{Definition}[section]

\renewcommand{\eqref}[1]{Eq.~(\ref{#1})}

\usepackage{amssymb}

\usepackage[style = alphabetic,citestyle=alphabetic,maxbibnames=99,sorting=nyt,natbib=true,backref=true]{biblatex}
\newcommand{\stoptocwriting}{%
  \addtocontents{toc}{\protect\setcounter{tocdepth}{-5}}}

\newcommand{\vertiii}[1]{{\left\vert\kern-0.25ex\left\vert\kern-0.25ex\left\vert #1
		\right\vert\kern-0.25ex\right\vert\kern-0.25ex\right\vert}}

\makeatletter

\DeclareMathOperator{\poly}{poly}

\DeclareMathOperator{\median}{median}

\newcommand{\norm}[1]{\left\lVert#1\right\rVert}

{\medskip}

{\hspace*{\fill}$\Box$\par\vspace{4mm}}
{\hspace*{\fill}$\Box$\par}

\newcommand{\cE}{\mathcal{E}}

\newcommand{\cI}{\mathcal{I}}
\newcommand{\cJ}{\mathcal{J}}

\newcommand{\cL}{\mathcal{L}}

\newcommand{\cO}{\mathcal{O}}

\newcommand{\cX}{\mathcal{X}}

\newcommand{\bE}{\mathbb{E}}

\newcommand{\bI}{\mathbb{I}}

\newcommand{\bP}{\mathbb{P}}
\newcommand{\bQ}{\mathbb{Q}}
\newcommand{\bR}{\mathbb{R}}

\title{\textbf{Sparse Mean Estimation in Adversarial Settings via Incremental Learning}}
\author{
Jianhao Ma\\
University of Pennsylvania\\
\texttt{jianhaom@wharton.upenn.edu}\\
\and 
Rui Ray Chen\\
Tsinghua University\\ 
\texttt{chenrui20@mails.tsinghua.edu.cn}
\and
Yinghui He\\
Princeton University\\
\texttt{yh0068@princeton.edu}\\
\and
Salar Fattahi\\
University of Michigan\\
\texttt{fattahi@umich.edu}
\\
\and
Wei Hu\\
University of Michigan\\
\texttt{vvh@umich.edu}
}

\input{math_commands.tex}

\begin{document}

\maketitle
\stoptocwriting

\input{main_text/abstract.tex}
\input{main_text/introduction.tex}

\input{main_text/related_work}

\input{main_text/overview_techniques}

\input{main_text/main_result}

\input{main_text/proof_sketch}

\input{main_text/simulation}

\input{main_text/deferred_proofs}

\input{main_text/conclusion}
\input{main_text/acknowledge}
\printbibliography

\appendix

\input{appendix/additional_simulation}
\end{document}

%% file: math_commands.tex
\usepackage{amsmath,amsfonts,bm}

\def\eqref#1{equation~\ref{#1}}

\def\1{\bm{1}}

\newcommand{\pr}{\mathrm{Pr}}

\def\vmu{{\bm{\mu}}}

\def\vg{g}
\def\vh{h}

\def\vu{u}
\def\vv{v}

\def\mI{I}

\def\mSigma{\Sigma}
\def\vmu{\mu}
\DeclareMathAlphabet{\mathsfit}{\encodingdefault}{\sfdefault}{m}{sl}
\SetMathAlphabet{\mathsfit}{bold}{\encodingdefault}{\sfdefault}{bx}{n}

\newcommand{\R}{\mathbb{R}}

\newcommand{\Var}{\mathrm{Var}}

\DeclareMathOperator*{\argmin}{arg\,min}

\DeclareMathOperator{\sign}{sign}
\DeclareMathOperator{\tildesign}{\widetilde{sign}}

%% file: main_text/abstract.tex
\begin{abstract}
    In this paper, we study the problem of sparse mean estimation under adversarial corruptions, where the goal is to estimate the $k$-sparse mean of a heavy-tailed distribution from samples contaminated by adversarial noise. Existing methods face two key limitations: they require prior knowledge of the sparsity level $k$ and scale poorly to high-dimensional settings. We propose a simple and scalable estimator that addresses both challenges. Specifically, it learns the $k$-sparse mean without knowing $k$ in advance and operates in near-linear time and memory with respect to the ambient dimension. Under a moderate signal-to-noise ratio, our method achieves the optimal statistical rate, matching the information-theoretic lower bound. Extensive simulations corroborate our theoretical guarantees.
At the heart of our approach is an {\it incremental learning} phenomenon: we show that a basic subgradient method applied to a nonconvex two-layer formulation with an $\ell_1$-loss can incrementally learn the $k$ nonzero components of the true mean while suppressing the rest. More broadly, our work is the first to reveal the incremental learning phenomenon of the subgradient method in the presence of heavy-tailed distributions and adversarial corruption.
\end{abstract}

%% file: main_text/introduction.tex
\section{Introduction}
\label{sec:intro}
Almost all statistical methods rely explicitly or implicitly on certain assumptions on the distribution of the data. In practice, however, these assumptions are only approximately satisfied, mainly due to the presence of heavy-tailed distributions and adversarial corruptions \citep{rousseeuw2011robust}.
To resolve these issues, the field of robust statistics has been developed to construct estimators that exhibit “\textit{insensitivity to small deviations from the (model) assumptions}” \cite[p.2]{huber2011robust}. Robust statistics has a long history with the fundamental work of John Tukey \citep{tukey1960survey, tukey1962future}, Peter Huber \citep{huber1964robust, huber1967under}, and Frank Hampel \citep{hampel1971general, hampel1974influence}. It has been applied across various domains, such as biology, finance, and computer science \citep{rousseeuw2011robust}.

Nonetheless, in high-dimensional scenarios, robust statistics contend with the \textit{curse of dimensionality}. Firstly, the majority of estimators in the literature demand exponential runtime with respect to data dimension. To resolve this problem, special attention has been devoted to \textit{algorithmic robust statistics}, which aims to design efficient algorithms for different tasks in the high-dimensional robust statistics (see the recent book \citep{diakonikolas2023algorithmic} and survey paper \citep{diakonikolas2019recent}).
Secondly, generic high-dimensional robust statistical tasks are often oblivious to the intrinsic structure of the data. As such, they rely on overly conservative sample sizes that have an undesirable dependency on the data dimension. 

In this paper, we aim to address these challenges for one of the most fundamental problems in robust statistics, namely \textit{robust sparse mean estimation}.
More specifically, given an $\epsilon$-corrupted set of samples from an unknown and possibly heavy-tailed distribution $\bP$ with a $k$-sparse mean $\vmu^{\star} = \bE[X] \in \bR^d$, our goal is to design a computationally and statistically efficient estimator $\hat{\vmu}$ of the mean $\vmu^{\star}$. 
Throughout this paper, we focus on the so-called {\it strong contamination model} \citep[Definition 1.6]{diakonikolas2023algorithmic} for the corruption in the data, which encompasses a variety of existing models, such as Huber's contamination model \citep{huber1964robust}.
\begin{definition}[Strong contamination model]
\label{assumption::comtamination}
Given a corruption parameter $\epsilon \in (0, \epsilon_0)$
and distribution $\bP$, the $\epsilon$-corrupted samples are generated as follows: (i) the algorithm specifies the number of samples $n$ and then $n$ i.i.d. samples are drawn from $\bP$. (ii) An arbitrarily powerful adversary then inspects the samples, removes $\epsilon n$ of them, and replaces them with arbitrary points. The resulting $\epsilon$-corrupted samples are given to the algorithm.
\end{definition}

Designing a statistically and computationally efficient estimator for the mean is highly nontrivial in this setting due to the following reasons.
First, contrary to the robust (dense) mean estimation, there is a conjectured \textit{computational-statistical tradeoff} \citep{diakonikolas2017statistical, brennan2019average, brennan2020reducibility} for the robust $k$-sparse mean estimation, which asserts that any efficient algorithm needs $\tilde{\Omega}(k^2)$ samples, while its statistically-optimal (but possibly inefficient) counterpart only requires $\tilde{\mathcal{O}}(k)$ samples. This conjecture has neither been proved nor refuted. Second, most existing mean estimators are designed for light-tailed distributions \citep{balakrishnan2017computationally, diakonikolas2019outlier, cheng2021outlier}. The only two efficient estimators available for heavy-tailed distributions \citep{diakonikolas2022outlier, diakonikolas2022robust}, however, are impractical for real-world applications, as they rely on computationally intensive techniques such as the ellipsoid algorithm and the sum-of-squares method.
A fundamental question thus arises:
\begin{quote}
    \textit{Can we design a practically efficient estimator for the robust sparse mean estimation problem that overcomes the conjectured computational-statistical tradeoff?} 
\end{quote}

In this work, we provide an affirmative answer to this question under moderate assumptions. Our proposed approach comprises two stages. In the first stage, we provide a coarse-grained estimation of the mean that is enough to identify the top-$k$ nonzero elements of the mean. In particular, we show that a simple subgradient method applied to a two-layer diagonal linear neural network with $\ell_1$-loss can identify the top-$k$ nonzero elements of the mean incrementally and sequentially while keeping the zero entries arbitrarily small. After the identification of the top-$k$ nonzero elements, in the second stage, we provide a finer-grained estimation of the nonzero elements of the mean by employing a generic robust mean estimator---such as those introduced in~\citet{diakonikolas2019recent, cheng2020high}---restricted to the top-$k$ nonzero elements, thereby reducing the effective dimension of the problem from $d$ to $k$.   
Our proposed approach achieves optimal statistical error, sample complexity, and computational cost under moderate assumptions. Furthermore, we demonstrate that these assumptions do not alter the inherent complexity of the problem, as evidenced by a matching information-theoretic lower bound.  \Cref{table::main} provides a summary of our results compared to the existing estimators. Our contributions are summarized below:

\begin{itemize}
    \item [-] {\bf{Overcoming the \textit{computational-statistical tradeoff}.}} We demonstrate that our algorithm can surpass the conjectured \textit{computational-statistical tradeoff} under additional conditions. At a high level, we require an $\epsilon$-dependent upper bound for the coordinate-wise third moment and a lower bound for the signal-to-noise ratio (SNR). Additionally, we demonstrate that our algorithm matches the information-theoretic lower bound under exactly the same conditions.
    \item [-] {\bf{Near-linear dependency on the dimension.}} The first stage of our algorithm is coordinate-wise decomposable and fully parallelizable. Therefore, it runs in $\tilde \cO(d)$ time and memory on a single thread, and in $\tilde \cO(d/K)$ time and $\tilde \cO(d)$ memory on $K$ threads. Moreover, the computational cost of the second stage of our algorithm is independent of $d$. In contrast, the existing robust sparse mean estimators have a poor dependency on $d$ (see \Cref{table::main}).
    \item [-] {\bf{No prior knowledge on the sparsity level.}} Our method does not require prior knowledge of the sparsity level $k$. In contrast, all existing methods for robust sparse mean estimation (in both light- and heavy-tailed settings) require knowledge of the sparsity level $k$. 
    \item [-] {\bf Superior practical performance.} Through extensive experiments, we show that, despite its simplicity, our method performs well across a broad class of heavy-tailed distributions, including those with unbounded variance.
\end{itemize}

\colorlet{shadecolor}{gray!20}
\begin{table*}[t!] 
    \centering
    \small
    \resizebox{0.8\linewidth}{!}{%
        \renewcommand{\arraystretch}{1.5}
        \begin{tabular}{|c|c|c|c|}
            \hline 
            \textbf{Algorithm} &  \textbf{$\ell_2$-error} & \textbf{Sample complexity} & \textbf{Running time} \\
            
            \hhline{|=|=|=|=|}
            Lower bound &  $ \Omega(\sqrt{\epsilon})$ & $\tilde\Omega(k/\epsilon)$ & - \\
            \hline
            \citep{depersin2020robust, prasad2020robust} &  $ \cO(\sqrt{\epsilon})$ & $\tilde\cO(k/\epsilon)$ & $\exp(d)$ \\
            \hline
            \citep{diakonikolas2022outlier} &  $ \cO(\sqrt{\epsilon})$ & $\tilde\cO\left(k^2/\epsilon\right)$ & $\poly(d)$ \\
            \hline
            \citep{diakonikolas2022robust} &  $ \cO(\sqrt{\epsilon})$ & $\tilde\cO(k^{\cO(1)}/\epsilon)$ & $\poly(d)$ \\
            \hline
            \rowcolor{shadecolor} \Gape[0pt][2pt]
            Ours (Stage 1)$^*$ & $\cO(\sqrt{k\epsilon})$ & $\tilde\cO(1/\epsilon)$ & $\tilde\cO(d)$\\
            \hline
            \rowcolor{shadecolor} \Gape[0pt][2pt]
            Ours (full)$^*$ & $\cO(\sqrt{\epsilon})$ & $\tilde\cO(k/\epsilon)$ & $\tilde\cO(d)$\\
            \hline
            
        \end{tabular}
    }
    \caption{\footnotesize Comparisons between different algorithms for robust sparse mean estimation. Here, $k$ represents the sparsity level, $d$ is the ambient dimension, and $\epsilon$ denotes the corruption ratio. We use $\tilde\Omega(\cdot)$ and $\tilde\cO(\cdot)$ to hide logarithmic factors. For simplicity, the dependency on the sample size is omitted in the above comparisons. $^*$Our algorithms require some mild assumptions as detailed in \Cref{thm::main}.
    }
    \label{table::main}
\end{table*}

%% file: main_text/related_work.tex
\section{Related Work}
\label{sec::related-work}
\paragraph{Robust (sparse) mean estimation.} Robust mean estimation is a fundamental problem in statistics, with its earliest work dating back to \citet{tukey1960survey, huber1964robust}. However, throughout its extensive history \citep{yatracos1985rates, donoho1988automatic, donoho1992breakdown, huber2011robust}, and even up to recent times \citep{lugosi2019robust, lugosi2019sub, depersin2020robust, prasad2020robust}, most statisticians have primarily focused on developing statistically optimal estimators, often overlooking the fact that these estimators can be computationally inefficient. It is only recently, following the seminal work of \citet{lai2016agnostic, diakonikolas2019robust}, that researchers have started to develop polynomial-time algorithms for robust mean estimation \citep{diakonikolas2017being, steinhardt2017resilience, cheng2019high} as well as other robust learning tasks, including robust PCA \citep{balakrishnan2017computationally} and robust regression \citep{chen2013robust}.

Robust sparse mean estimation, as a distinct variant, has attracted considerable attention, particularly in extremely high-dimensional settings. However, the situation for robust sparse mean estimation is more nuanced compared to the dense case. Firstly, unlike the dense case, there is a conjectured \textit{computational-statistical tradeoff} \citep{diakonikolas2017statistical, brennan2019average, brennan2020reducibility}, suggesting that efficient algorithms demand a qualitatively larger sample complexity than their inefficient counterparts. In particular, there is evidence that such a tradeoff is unavoidable for Stochastic Query (SQ) algorithms \citep{diakonikolas2017statistical}. On the other hand, most prior works have primarily concentrated on the light-tailed setting \citep{balakrishnan2017computationally, diakonikolas2019outlier, cheng2021outlier}. Researchers have only recently addressed the heavy-tailed setting using stability-based approaches \citep{diakonikolas2022outlier} and sum-of-squares methods \citep{diakonikolas2022robust}. While these algorithms are polynomial-time, they may not be practical when dealing with high-dimensional settings.

\paragraph{Incremental learning.} Over the past few years, it has been shown practically and theoretically that gradient-based methods tend to explore the solution space in an incremental order of complexity, ultimately favoring low-complexity solutions in numerous machine learning tasks \citep{gissin2019implicit, ma2025implicit}. This phenomenon is known as \textit{incremental learning}. Specifically, researchers have investigated incremental learning in various contexts, such as
matrix factorization and its variants \citep{li2020towards, ma2022behind, jin2023understanding}, tensor factorization \citep{razin2021implicit, razin2022implicit, ma2022behind}, deep linear networks \citep{arora2019implicit,gidel2019implicit, li2021implicit, mablessing}, and general neural networks \citep{hu2020surprising, frei2022implicit}. In essence, incremental learning is believed to be crucial for understanding the empirical success of optimization and generalization in contemporary machine learning \citep{gissin2019implicit}. However, to the best of our knowledge, its emergence in adversarial settings remains unexplored.

\paragraph{Notation:} 
We use the notations $a(n)\lesssim b(n)$ and $a(n) = \cO(b(n))$ to denote $a(n)\leq Cb(n)$, for a universal constant $C$ and sufficiently large $n$. Similarly, the notations $a(n)\gtrsim b(n)$ and $a(n)=\Omega(b(n))$ are used to denote $a(n)\geq Cb(n)$, for a universal constant $C$ and sufficiently large $n$.
The notation $a=\Theta(b)$ is used to denote $a=\cO(b)$ and $b=\cO(a)$. Moreover, the notation $a(n) = o(b(n))$ implies that $\lim_{n\to+\infty}a(n)/b(n)=0$. The $\sign(\cdot)$ function is defined as $\sign(x)=x/|x|$ if $x\neq 0$, and $\sign(0)=[-1,1]$. We also define $\tildesign(x)=x/|x|$ if $x\neq 0$, and $\tildesign(0)=0$. Given a set $\cX$, the indicator function $\bI_{\cX}(\cdot)$ is defined as $\bI_{\cX}(x)=1$ if $x\in \cX$, and $\bI_{\cX}(x)=0$ otherwise. Similarly, and with a slight abuse of notation, for an event $\cE$, we define the indicator function $\bI(\cE) = 1$ if $\cE$ occurs, and $\bI(\cE) = 0$ otherwise.  We denote $[n]:=\{1,2,\cdots,n\}$. For two functions $f,g:\mathbb{R}^d\to \mathbb{R}$, we define $\norm{f-g}_{\infty} = \sup_{x\in \mathbb{R}^d}|f(x)-g(x)|$. For two vectors $x,y\in\R^d$, their Hadamard product is defined as $x\odot y=[x_1y_1\ \cdots\ x_dy_d]^\top$. For a vector $x\in\R^d$, we define $x^2=[x_1^2, \cdots, x_d^2]^{\top}$. For a vector $x\in \R^d$ and index set $I$ with size $k$, the notation $[x]_I\in \R^k$ refers to the projection of $x$ onto $I$. Moreover, we define $x\wedge y=\min\{x, y\}$. We represent mixtures of probability distributions as linear combinations of their corresponding density functions. For example, given two distributions $\bP_1$ and $\bP_2$ and a scalar $0 \leq \epsilon \leq 1$, we define the mixture $\bP_3 = (1-\epsilon)\bP_1 + \epsilon \bP_2$. A sample from $\bP_3$ is drawn from $\bP_1$ with probability $1-\epsilon$ and from $\bP_2$ with probability $\epsilon$.

%% file: main_text/overview_techniques.tex
\section{Overview of Our Approach}
\label{sec::overview-techniques}
To lay the groundwork, we begin by introducing the standard \textit{median-of-means} (\textsf{MoM}) estimator \citep{nemirovskij1983problem, jerrum1986random, alon1996space} originally designed for estimating the mean of a one-dimensional random variable. \textsf{MoM} estimator serves as a cornerstone for more sophisticated methods as detailed in \citet{lugosi2019sub, prasad2020robust1, lecue2020robust, diakonikolas2022outlier}. 

\begin{definition}[Median-of-means estimator for one-dimensional case]
    \label{def::mom}
    Given a set of $\epsilon$-corrupted samples $S=\{X_1, \cdots, X_n\}\subset \mathbb{R}$, we first partition them into $J$ subgroups $S_1, \cdots, S_J$ with equal sizes, where we assume $n$ is divisible by $J$ for simplicity. We then calculate the sample mean for each subgroup, i.e., $\bar X_j=\frac{1}{B}\sum_{i\in S_j}X_i$ where $B=n/J$. Subsequently, the median-of-means (\textsf{MoM}) estimator is obtained by taking the median of the sample means $\bar X_1, \cdots, \bar X_J$, i.e., $\hat{\mu}_{\textsf{MoM}}=\median\left\{\bar X_1, \cdots, \bar X_J\right\}$.
\end{definition}
Alternatively, the \textsf{MoM} estimator can be expressed as the minimizer of the following $\ell_1$-loss:
\begin{equation}
    \hat{\mu}_{\textsf{MoM}}=\argmin_{\mu\in \bR} \frac{1}{J}\sum_{j=1}^{J}\left|\bar X_j-\mu\right|.
\end{equation}
By appropriately selecting the number of subgroups $J$, it can be shown that the \textsf{MoM} estimator matches the information-theoretic lower bound $\Omega(\sigma\sqrt{\epsilon})$ for heavy-tailed distributions under the strong contamination model (\Cref{assumption::comtamination}).
\begin{sloppypar}
\begin{proposition}[One-dimensional \textsf{MoM} estimator]
    \label{prop::mom-1d}
    Consider a corruption parameter $\epsilon$, a failure probability $\delta$, and a set $S$ of $n$ many $\epsilon$-corrupted samples from a distribution $\bP$ with mean $\mu^{\star}\in \mathbb{R}$ and variance $\bE[(X-\mu^{\star})^2]\leq \sigma^2$. Suppose that $n\gtrsim \log(1/\delta)/\epsilon$. Then, upon choosing the number of subgroups $J=\Theta(\lceil \epsilon n\rceil+\log(1/\delta))$, with probability at least $1-\delta$ over the sample set $S$, the \textsf{MoM} estimator $\hat{\mu}_{\textsf{MoM}}$ satisfies $|\hat{\mu}_{\textsf{MoM}}-\mu^{\star}|=\cO\left(\sigma\sqrt{\epsilon}\right)$.
\end{proposition}
\end{sloppypar}
A more precise statement of \Cref{prop::mom-1d} and its proof are presented in Appendix \ref{sec::mom}.

Naively applying \textsf{MoM} estimator to different coordinates of a high-dimensional random variable leads to an undesirable dependency on the dimension $d$. More precisely, the coordinate-wise \textsf{MoM}, which corresponds to the solution to the following convex optimization

\begin{align}
    \hat{\vmu}_{\textsf{MoM}}=\argmin_{\vmu\in \bR^d} \cL_{\textsf{cvx}}(\vmu):=\frac{1}{J}\sum_{j=1}^{J}\norm{\bar X_j-\vmu}_1,
    \tag{\textsc{cvx}}\label{eq::convex}
\end{align}
suffers from a suboptimal error rate of $\norm{\hat{\vmu}_{\textsf{MoM}}-\mu^\star}_2 = \cO(\sigma\sqrt{d\epsilon})$ (see \Cref{thm::mom-high-dim} in Appendix \ref{sec::mom}). This error is unavoidable for the \textsf{MoM} estimator since the coordinate-wise error $\cO(\sigma\sqrt{\epsilon})$ is uniformly distributed across each coordinate. An alternative approach, the geometric \textsf{MoM} \citep{minsker2015geometric}, which replaces the $\norm{\cdot}_1$ in ~\ref{eq::convex} by $\norm{\cdot}_2$, also suffers from a similar error.

\paragraph{Two-layer model} To address the above issue, we model the mean $\vmu$ as a two-layer model $\vu^2-\vv^2$ for $u,v\in\R^d$, and obtain $(u,v)$ by minimizing the following nonconvex $\ell_1$-loss
\begin{align}
    \min_{\vu, \vv\in \bR^d}\cL_{\textsf{ncvx}}(\vu, \vv)=\frac{1}{2J}\sum_{j=1}^{J}\norm{\bar X_j - \left(\vu^2-\vv^2\right)}_1.
    \tag{\textsc{ncvx}}\label{eq::nonconvex}
\end{align}
To solve this optimization problem, we propose a subgradient method (SubGM) with small initialization $\vu(0)=\vv(0)=\alpha \vec{1}$, where $\vec{1}=[1, \cdots, 1]^{\top}\in \bR^d$ and $\alpha>0$ is a sufficiently small factor. At each iteration, SubGM updates the solution as
\begin{align}
		\vu(t+1)&=\vu(t)-\eta \vg(t) \quad\text{where}\quad\vg(t)\in \partial_\vu \cL_{\textsf{ncvx}}(\vu(t), \vv(t)),\nonumber\\
        \vv(t+1)&=\vv(t)-\eta \vh(t)\quad\text{where}\quad\vh(t)\in \partial_\vv \cL_{\textsf{ncvx}}(\vu(t), \vv(t)).\tag{\textsc{SubGM}}
\label{eq::2}
\end{align}
Here, $\eta>0$ is the stepsize, and $\partial_\vu \cL_{\textsf{ncvx}}(\vu, \vv)$ and $\partial_\vv \cL_{\textsf{ncvx}}(\vu, \vv)$ indicate the (Clarke) subdifferentials of $\cL_{\textsf{ncvx}}$, defined as:
\begin{align}
    \partial_\vu \cL_{\textsf{ncvx}}(\vu, \vv) &=\frac{1}{J}\sum_{j=1}^{J}\sign(\vu^2-\vv^2-\bar X_j)\odot\vu,\\
    \partial_\vv \cL_{\textsf{ncvx}}(\vu, \vv) &=-\frac{1}{J}\sum_{j=1}^{J}\sign(\vu^2-\vv^2-\bar X_j)\odot\vv.
\end{align}
The detailed implementation of our proposed algorithm is presented in \Cref{alg:main}.

\begin{algorithm}[h]
	\caption{\textsc{Robust sparse mean estimation via incremental learning}}\label{alg:main}
	\begin{algorithmic}[1]
        \Statex \textbf{Input:} dataset $S$, corruption parameter $\epsilon$, failure probability $\delta$, initialization scale $\alpha$, stepsize $\eta$, and iteration time $T\in \left[\frac{2}{\eta}\log(1/\alpha), \frac{6}{\eta}\log(1/\alpha)\right]$.
        \Stage
        \State {\bf Pre-processing:} Divide the dataset into $J$ equal subgroups $S_1, \cdots, S_J$, where $J=100\lceil\epsilon n\rceil$. Calculate the sample means $\bar X_j=\frac{J}{n}\sum_{i\in S_j}X_i$.
		\State {\bf Initialization:}  $\vu(0)=\vv(0)=\alpha.\vec{1}$
        \For{$t=1,\cdots,T$}
		\State Update $\vu(t), \vv(t)$ via \ref{eq::2}.
		\EndFor
        \State {\bf Identification of top-$k$ elements:} Calculate $I=\left\{i\in [d]:|u^2_i(T)-v^2_i(T)|\geq \alpha\right\}$.
        \State \textbf{Return} $\hat\vmu(T)=\vu^2(T)-\vv^2(T)$.
        \EndStage
        \Stagetwo
    \State Consider the projected dataset $S_k = \{[X_i]_{I}:X_i\in S\}$ and apply an existing robust mean estimator (e.g. those introduced in \citet{diakonikolas2019recent, cheng2020high}) to $S_k$.
    \EndStagetwo
\end{algorithmic}
\end{algorithm}
\begin{figure*}\centering
    \begin{centering}
    \subfloat[nonconvex]{
    {\includegraphics[width=0.35\linewidth]{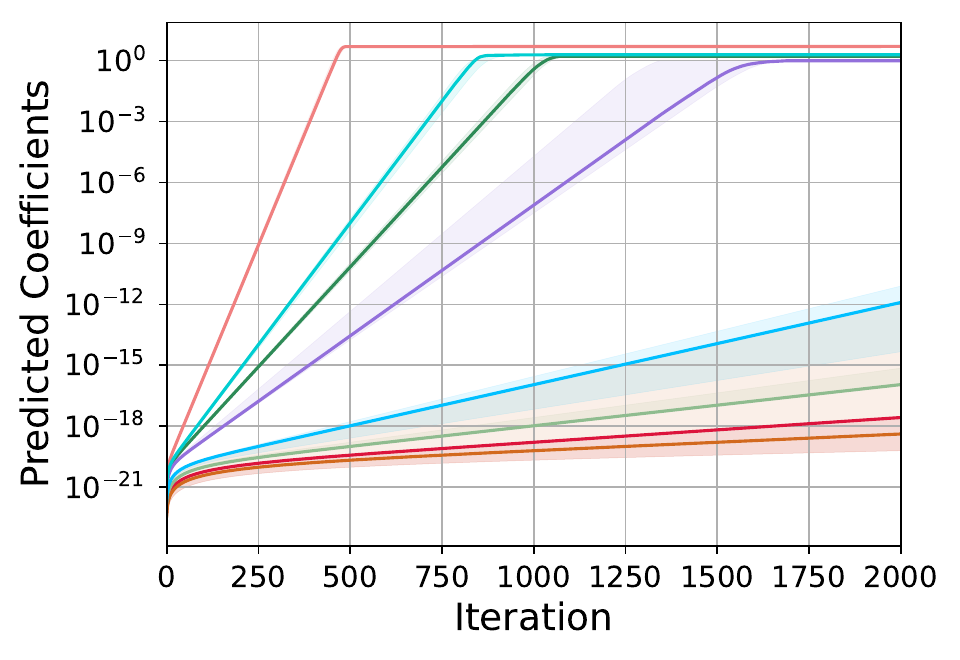}}\label{fig::nonconvex}}
    \subfloat[convex]{
    {\includegraphics[width=0.35\linewidth]{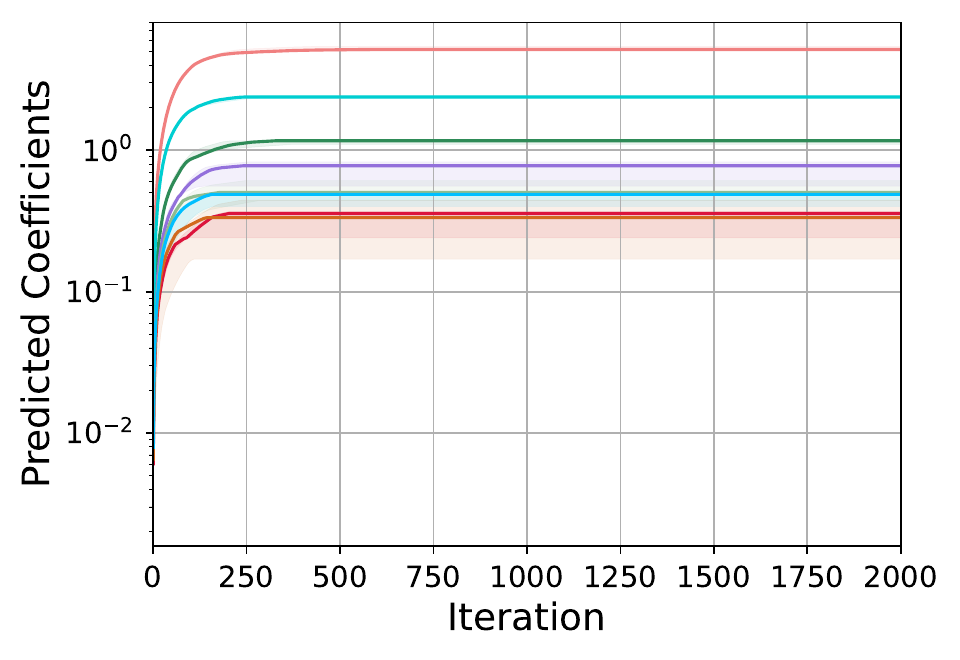}}\label{fig::convex}}
    \end{centering}
    \caption{\footnotesize The predicted coefficients for \ref{eq::nonconvex} and \ref{eq::convex}. We run the subgradient method with stepsize $\eta=0.07$ in both settings. We initialize \ref{eq::nonconvex} with $\alpha=1\times 10^{-10}$ and use a zero initialization for \ref{eq::convex}. We generate the inliers with $d=500, n=2000$ from a lognormal distribution with a variance of $3.3$ and a $k$-sparse mean with $k=4$ and nonzero elements $[5, 2, 2, 1.5]$. The corruption rate is $0.05$ and the outliers are generated from a Cauchy distribution with a mean of $20$ and variance of $50$.}
    \label{fig::convex-vs-nonconvex}
\end{figure*}

Our key contribution is to reveal the emergence of incremental learning: we show that SubGM with small initialization learns the nonzero components (signals) long before overfitting the zero components (residuals) to noise. Consequently, there exists a wide range of iterations within which the signals are in the order of $\Omega(1)$ while the residuals remain in the order of $\cO(\alpha)$ (see \Cref{fig::nonconvex}). Remarkably, we show that this interval only depends on the stepsize $\eta$ and the initialization scale $\alpha$, and  it can be widened by reducing these user-defined parameters. In stark contrast, differentiating between the signals and residuals is challenging in the convex setting (\ref{eq::convex}) precisely due to the lack of incremental learning, as shown in \Cref{fig::convex}. After successfully identifying the locations of the top-$k$ elements, we can employ existing robust mean estimation techniques \citep{diakonikolas2019recent, cheng2020high} on the dataset projected onto the recovered support to further improve the estimation of the top-$k$ nonzero elements.

%% file: main_text/main_result.tex
\section{Main Result}
\label{sec::main_result}
In this section, we present the theoretical guarantees for \Cref{alg:main}. We begin by analyzing the first stage of the algorithm, which focuses on recovering the support of the true mean.
\subsection{Stage 1: Identification of Support via Coarse-grained Estimation}
\label{sec::stage1}
We denote $\mu_{\max}^{\star} = \max_i\{|\mu_i^\star|\}$ and $\mu_{\min}^{\star} = \min_{i}\{|\mu_i^\star|: \mu_i^\star\not=0\}$. Our main theorem is presented next.
\begin{theorem}[Convergence guarantee for SubGM]
    \label{thm::main}
    Let $\bP$ be a distribution on $\bR^d$ with an unknown $k$-sparse mean $\vmu^{\star}$, unknown covariance matrix $\mSigma \preceq \sigma^2\mI$, and unknown coordinate-wise third moment satisfying $\bE[|X_i-\mu_i^{\star}|^3]\lesssim \sigma^3/\sqrt{\epsilon}, \forall 1\leq i\leq d$. Suppose a sample set of size $n\gtrsim\log(d/\delta)/\epsilon$ is collected according to the strong contamination model (\Cref{assumption::comtamination}) with corruption parameter $\epsilon$. Upon setting the stepsize $\eta\leq \sigma\sqrt{\epsilon}/\mu_{\max}^{\star}$ and the initialization scale $0<\alpha\lesssim \sigma\sqrt{\epsilon/d}\wedge \mu_{\max}^{\star -5}$ in \Cref{alg:main}, with a probability of at least $1-\delta$, the following statements hold for any iteration $\frac{2}{\eta}\log(1/\alpha)\leq T\leq \frac{6}{\eta}\log(1/\alpha)$:
\begin{itemize}
    \item \textbf{$\ell_2$-error.} The $\ell_2$-error is upper-bounded by
    \begin{equation}
        \norm{\hat\vmu(T)-\vmu^{\star}}_2\lesssim \sigma\sqrt{k\epsilon}.
    \end{equation}
    \item \textbf{Identification of the top-$k$ elements.} If we additionally have $\epsilon\lesssim\mu_{\min}^{\star 2}/\sigma^2$, then we obtain 
    \begin{equation}
        \begin{aligned}
            |\hat{\mu}_i(T)|&\gtrsim\sigma\sqrt{\epsilon}, \quad &\text{where }\mu_i^{\star}\neq 0,\\
            |\hat{\mu}_i(T)|&\lesssim \alpha, \quad &\text{where }\mu_i^{\star}= 0.
        \end{aligned}
    \end{equation}
\end{itemize}
\end{theorem}

{\bf Comparison to the existing results.} 
Simply applying coordinate-wise \textsf{MoM} estimator results in an $\ell_2$-error rate $\cO(\sigma\sqrt{d\epsilon})$, which is considerably worse than our result when $k\ll d$. On the other hand, to guarantee a correct support recovery, the previous efficient estimators rely on prior knowledge of $k$, while the coordinate-wise \textsf{MoM} requires an accurate value of $\mu_{\min}^{\star}$ to separate the signals from residuals (as evidenced by \Cref{fig::convex}).
In contrast, our proposed algorithm only requires a lower bound $\hat\mu_{\min}\leq \mu_{\min}^{\star}$ to differentiate the signals from residuals; in fact, this lower bound can be arbitrarily small (i.e., conservative) provided that the initialization scale is chosen as $\alpha\ll \hat\mu_{\min}$. We also highlight that, much like other existing estimators under the strong contamination model, our estimator requires prior knowledge of the corruption parameter $\epsilon$ (or its upper bound).

{\bf Proof sketch.} We next provide the proof sketch of the above theorem, deferring its details to \Cref{sec::deferred-proofs}. Specifically, we analyze the coordinate-wise dynamic $\hat{\mu}_i(t)=u_i^2(t)-v_i^2(t)$ for some $1\leq i\leq d$. Without loss of generality, we assume $\mu_i^{\star}\geq 0$. Upon defining $\beta_i(t)=\frac{1}{J}\sum_{j=1}^{J}\tildesign(\bar X_{j, i}-\hat{\mu}_i(t))$, the update rules for $\vu_i(t)$ and $\vv_i(t)$ can be written as 
\begin{equation}
    \begin{aligned}
        u_i(t+1)=\left(1+\eta\beta_i(t)\right)u_i(t), \quad v_i(t+1)=\left(1-\eta\beta_i(t)\right)v_i(t).
    \end{aligned}
\end{equation}
Based on the above update rules, $\beta_i(t)$ controls the growth rate of the dynamics. Indeed,  during the initial iterations, we have $\hat{\mu}_i(t)\approx\hat{\mu}_i(0) = u_i^2(0)-v_i^2(0) = 0$, which in turn implies that $\beta_i(t)\approx \beta_i(0)$. Consequently, the dynamics of $\vu_i(t)$ and $\vv_i(t)$ can be well approximated using the following exponential functions
\begin{equation}
    \vu_i(t)\approx (1+\eta\beta_i(0))^t\alpha, \quad \vv_i(t)\approx (1-\eta\beta_i(0))^t\alpha.
\end{equation}
Therefore, to analyze the behaviors of $\vu_i(t)$ and $\vv_i(t)$, it suffices to characterize the magnitude of $\beta_i(0)$ for different coordinates. To achieve this, we define $\cJ_{\text{clean}}$ as the index set of the subgroups $[J]$ that do not contain any outliers, and denote its complement as $\cJ_{\text{outlier}}=[J]\backslash\cJ_{\text{clean}}$. We have
\begin{align}
    \beta_i(0)&=\frac{1}{J}\sum_{j\in \cJ_{\text{clean}  }}\tildesign(\bar X_{j, i})\pm \frac{\left|\cJ_{\text{outlier}}\right|}{J}\tag{denote $\delta=\frac{\left|\cJ_{\text{outlier}}\right|}{J}$}\\
    &\approx (1-\delta)\bE\left[\tildesign(\bar X_{j, i})\right]\pm \delta\tag{for sufficiently large $\cJ_{\text{clean}  }$}\\
    &=(1-\delta)\left(1-2\pr\left(\bar X_{j, i}-\mu_i^{\star}\leq-\mu_i^{\star}\right)\right)\pm \delta\nonumber\\
    &\approx (1-\delta)(1-2\Phi(-\mu_i^{\star}\cdot\sqrt{B\Var(X)}))\pm\delta. \tag{due to finite-sample central limit theorem}
\end{align}
Here, $B=n/J$ is the size of each subgroup, and $\Phi(\cdot)$ represents the cumulative distribution function (CDF) of the standard Gaussian distribution. Let us define $\cI_{\text{residual}} = \{i: \mu_i^\star = 0\}$ and $\cI_{\text{signal}} = \{i: \mu_i^\star \not= 0\}$. Based on the above characterization of $\beta_i(0)$, for all $i\in \cI_{\text{residual}}$, we have $1-2\Phi(-\mu_i^{\star}\cdot\sqrt{B\Var(X)})=0$, which in turn implies $\beta_i(0)\approx \pm\delta$. Furthermore, by setting $J=C \lceil\epsilon n\rceil$ with a suitably large constant $C$, $B=n/J\geq 1/(C\epsilon)$ can be made sufficiently large given a sufficiently small $\epsilon$. This ensures that $\beta_i(0)\approx \Omega(1-\delta)\pm \delta$ for all $i\in \cI_{\text{signal}}$. On the other hand, we have $\delta\leq \lceil\epsilon n\rceil/J\leq 1/C$ since $\left|\cJ_{\text{outlier}}\right|\leq \lceil\epsilon n\rceil$.
As a result, $|\beta_i(0)|$ can be made arbitrarily small for all $i\in \cI_{\text{residual}}$ and $\beta_i(0)= \Omega(1)$ for all $i\in \cI_{\text{signal}}$. This discrepancy in the growth rates of $u_i(t)$ and $v_i(t)$ enables our algorithm to separate the signals from residuals within just a few iterations. In \Cref{sec::deferred-proofs}, we provide a more delicate analysis of the dynamics, showing that for all $T\in [\frac{2}{\eta}\log(1/\alpha), \frac{6}{\eta}\log(1/\alpha)]$ we have
\begin{equation}
    \begin{aligned}
        &\, u_i^2(t)-v_i^2(t)= \mu_i^{\star}\pm \cO(\sigma \sqrt{\epsilon}), \quad &&\text{for } i\in\cI_{\text{signal}},\\
        &\left|u_i^2(t)-v_i^2(t)\right|=\poly(\alpha)\ll \sigma \sqrt{\epsilon}, \quad &&\text{for }i\in \cI_{\text{residual}}.
    \end{aligned}
\end{equation}
The above equation sheds light on the key difference between \ref{eq::nonconvex} and \ref{eq::convex}: unlike ~\ref{eq::convex} where the error is equally distributed across different coordinates, the error in~\ref{eq::nonconvex} is primarily distributed among the signals, while the error in the residuals can be kept arbitrarily small by a proper choice of the initialization scale $\alpha$. 
This implies that, if the signals are sufficiently larger than the induced error, i.e., $|\mu_i^{\star}|\gtrsim \sigma\sqrt{\epsilon}, \forall i\in \cI_{\text{signal}}$, our algorithm can successfully identify the signals.

\subsection{Stage 2: Achieving Optimal Rate on the Support via Fine-grained Estimation}
As illustrated in \Cref{sec::stage1}, a direct application of SubGM leads to an estimation error of $\cO(\sqrt{k\epsilon})$. In this section, we show that this error can be further improved once the support of the mean is identified correctly. Our key insight is that once the support of the mean is recovered, we can reduce the problem to a robust \textit{dense} mean estimation defined \textit{only} over the recovered support. 
Under such a regime, existing estimators designed for robust dense mean estimation \citep{diakonikolas2019recent, cheng2020high} can be employed to further reduce the estimation error. 
\begin{proposition}[Adapted from Proposition 1.6 in \citet{diakonikolas2020outlier}]
\label{prop::k-dimensional}
    Let $\bP$ be a distribution on $\bR^k$ with an unknown mean $\vmu^{\star}$ and unknown covariance matrix $\mSigma \preceq \sigma^2\mI$. Suppose a sample set of size $n$ is collected according to the strong contamination model (\Cref{assumption::comtamination}) with corruption parameter $\epsilon< 1/2$. Then, there exists an algorithm that runs in $\cO(kn)$ time and memory and, with a probability of at least $1-\delta$, outputs an estimator $\hat{\mu}$ that satisfies 
    $$\norm{\hat{\vmu}-\vmu^{\star}}_2\lesssim\sigma\sqrt{\epsilon}+\sigma\sqrt{k/n}+\sigma\sqrt{\log(1/\delta)/n}.$$
\end{proposition}
Equipped with the above result, we next provide an end-to-end guarantee for our full algorithm.
\begin{theorem}[Guarantee for the full algorithm]
    \label{thm::main2}
    Let $\bP$ be a distribution on $\bR^d$ satisfying the conditions in \Cref{thm::main}. Suppose a sample set of size $n\gtrsim ({k+\log(d/\delta)})/{\epsilon}$ is collected according to the strong contamination model (\Cref{assumption::comtamination}) with corruption parameter $\epsilon\lesssim \mu_{\min}^2/\sigma^2$.
    Then, with the choice of $\alpha=\Theta(\sigma\sqrt{\epsilon/d}\wedge \mu_{\max}^{\star -5})$ and $\eta=\Theta(\sigma\sqrt{\epsilon}/\mu_{\max}^{\star})$, our full algorithm runs in $\cO(nd\log(d))$ time and memory and, with a probability of at least $1-\delta$, outputs an estimate $\hat{\vmu}$ that satisfies
\begin{equation}
    \norm{\hat{\vmu}-\vmu^{\star}}_2\lesssim \sigma\sqrt{\epsilon}. \label{eq::6}
\end{equation}
\end{theorem}
Upon setting the sample size $n=\Theta\left( ({k+\log(d/\delta)})/{\epsilon}\right)$, our proposed two-stage method runs in $\tilde\cO(dk)$ time and memory and returns a solution with an error in the order of $\cO(\sigma\sqrt{\epsilon})$. Our next theorem shows that this error is indeed information-theoretically optimal up to a constant factor and thus cannot be improved.
\begin{theorem}[Information-theoretic lower bound]
\label{lem::information-bound}
There exists a distribution $\bP$ with $k$-sparse mean $\vmu^{\star}$, covariance matrix $\mSigma \preceq \sigma^2\mI$, and coordinate-wise third moment satisfying $\bE[|X_i-\mu_i^{\star}|^3]\lesssim \sigma^3/\sqrt{\epsilon}, \forall 1\leq i\leq d$ such that, given any arbitrarily large sample set collected according to the strong contamination model (\Cref{assumption::comtamination}) with corruption parameter $\epsilon$, no algorithm can estimate the mean $\mu^\star$ with $\ell_2$-error $o(\sigma\sqrt{\epsilon})$.
\end{theorem}

{\bf Comparison to the existing lower bounds.} To achieve the optimal error rate, the sample complexity of our method scales linearly with the sparsity level $k$. A careful reader may realize that our sample complexity is unexpectedly smaller than the optimal sample complexity $\Omega((k\log(d/k)+\log(d/\delta))/\epsilon)$ introduced in \citet{lugosi2019near} when $k$ is sufficiently small. {This is due to the additional assumptions we impose on the coordinate-wise third moment of the distribution and the corruption parameter $\epsilon$.} On the other hand, it is recently shown in \citet{diakonikolas2019recent, prasad2020robust} that under the bounded third moment, the dependency of the estimation error on $\epsilon$ can be improved from $\epsilon^{1/2}$ to $\epsilon^{2/3}$. {Our worse dependency on $\epsilon$ is due to our more relaxed assumption on the third moment: unlike the assumptions made in \citet{diakonikolas2019recent, prasad2020robust}, our imposed upper bound on the third moment is inversely proportional to $\sqrt{\epsilon}$. Consequently, the imposed upper bound can get arbitrarily large with a smaller corruption parameter. In this extreme case where $\epsilon\to 0$, this condition can be dropped all together.}

%% file: main_text/proof_sketch.tex
\section{Proof Sketch}
\label{sec::proof-sketch}
In this section, we provide a proof sketch for the dynamics of SubGM (Theorem~\ref{thm::main}). To streamline the presentation, we keep our arguments at a high level; a more detailed proof is deferred in the supplementary materials. We analyze the coordinate-wise dynamic $\hat{\mu}_i(t)=u_i^2(t)-v_i^2(t)$ for $1\leq i\leq d$. Without loss of generality, we assume $\mu_i^{\star}\geq 0$. Upon defining $\beta_i(t)=\frac{1}{J}\sum_{j=1}^{J}\tildesign(\bar X_{j, i}-\hat{\mu}_i(t))$, the update rules for $\vu_i(t)$ and $\vv_i(t)$ can be written as 
\begin{equation}
    \begin{aligned}
        u_i(t+1)=\left(1+\eta\beta_i(t)\right)u_i(t), \quad v_i(t+1)=\left(1-\eta\beta_i(t)\right)v_i(t).
    \end{aligned}
\end{equation}
Based on the above update rules, $\beta_i(t)$ controls the growth rate of the dynamics. Indeed,  during the initial iterations, we have $\hat{\mu}_i(t)\approx\hat{\mu}_i(0) = u_i^2(0)-v_i^2(0) = 0$, which in turn implies that $\beta_i(t)\approx \beta_i(0)$. Consequently, the dynamics of $\vu_i(t)$ and $\vv_i(t)$ can be well approximated using the following exponential functions
\begin{equation}
    \vu_i(t)\approx (1+\eta\beta_i(0))^t\alpha, \quad \vv_i(t)\approx (1-\eta\beta_i(0))^t\alpha.
\end{equation}

Therefore, to analyze the behaviors of $\vu_i(t)$ and $\vv_i(t)$, it suffices to characterize the magnitude of $\beta_i(0)$ for different coordinates. To achieve this, we define $\cJ_{\text{clean}}$ as the index set of the subgroups $[J]$ that do not contain any outliers and $\cJ_{\text{outlier}}=[J]-\cJ_{\text{clean}}$. Consequently, we have
\begin{align}
    \beta_i(0)&=\frac{1}{J}\sum_{j\in \cJ_{\text{clean}  }}\tildesign(\bar X_{j, i})\pm \frac{\left|\cJ_{\text{outlier}}\right|}{J}\tag{denote $\delta=\frac{\left|\cJ_{\text{outlier}}\right|}{J}$}\\
    &\approx (1-\delta)\bE\left[\tildesign(\bar X_{j, i})\right]\pm \delta\tag{due to concentration bound}\\
    &=(1-\delta)\left(1-2\pr\left(\bar X_{j, i}-\mu_i^{\star}\leq-\mu_i^{\star}\right)\right)\pm \delta\nonumber\\
    &\approx (1-\delta)(1-2\Phi(-\mu_i^{\star}\cdot\sqrt{B/\Var(X)}))\pm\delta. \tag{due to finite-sample central limit theorem}
\end{align}
Here $\Phi(\cdot)$ represents the cumulative distribution function (CDF) of the standard Gaussian distribution. Let us define $\cI_{\text{residual}} = \{i: \mu_i^\star = 0\}$ and $\cI_{\text{signal}} = \{i: \mu_i^\star \not= 0\}$. Based on the above characterization of $\beta_i(0)$, we have $\beta_i(0)\approx \pm\delta$ for all $i\in \cI_{\text{residual}}$. Furthermore, by setting $J=C \lceil\epsilon n\rceil$ with a suitably large constant $C$, we can ensure $\beta_i(0)\approx \Omega(1-\delta)\pm \delta$ for all $i\in \cI_{\text{signal}}$ because $B=n/J\geq 1/(C\epsilon)$ can be made sufficiently large given a sufficiently small $\epsilon$. On the other hand, we have $\delta\leq \lceil\epsilon n\rceil/J\leq 1/C$ since $\left|\cJ_{\text{outlier}}\right|\leq \lceil\epsilon n\rceil$.
As a result, $|\beta_i(0)|$ can be made arbitrarily small for all $i\in \cI_{\text{residual}}$ and $\beta_i(0)= \Omega(1)$ for all $i\in \cI_{\text{signal}}$. This discrepancy in the growth rates of $u_i(t)$ and $v_i(t)$ enables our algorithm to separate the signals from residuals within just a few iterations. In the supplementary materials, we provide a more delicate analysis of the dynamics, showing that for all $T\in [\frac{2}{\eta}\log(1/\alpha), \frac{6}{\eta}\log(1/\alpha)]$ we have
\begin{equation}
    \begin{aligned}
        &\, u_i^2(t)-v_i^2(t)= \mu_i^{\star}\pm \cO(\sigma \sqrt{\epsilon}), \quad &&\text{for } i\in\cI_{\text{signal}},\\
        &\left|u_i^2(t)-v_i^2(t)\right|=\poly(\alpha)\ll \sigma \sqrt{\epsilon}, \quad &&\text{for }i\in \cI_{\text{residual}}.
    \end{aligned}
\end{equation}
The above equation sheds light on the key difference between \ref{eq::nonconvex} and \ref{eq::convex}: unlike ~\ref{eq::convex} where the error is equally distributed across different coordinates, the error in~\ref{eq::nonconvex} is primarily distributed among the signals, while the error in the residuals can be kept arbitrarily small by a proper choice of the initialization scale $\alpha$. 
This implies that, if the signals are sufficiently larger than the induced error, i.e., $|\mu_i^{\star}|\gtrsim \sigma\sqrt{\epsilon}, \forall i\in \cI_{\text{signal}}$, our algorithm can successfully identify the signals.

%% file: main_text/simulation.tex
\section{Simulation}
\label{sec::simulation}
In this section, we present numerical simulations to corroborate the theoretical results established in \Cref{sec::main_result}. Further implementation details, together with additional simulation studies, are deferred to the appendix. The complete codebase is publicly accessible at \url{https://github.com/ying-hui-he/Robust_mean_estimation}.

\begin{figure*}
    \begin{centering}
    \subfloat[Success rate of identification of top-$k$ elements in Stage 1 with varying corruption ratio $\epsilon$]{
    {\includegraphics[width=0.9\linewidth]{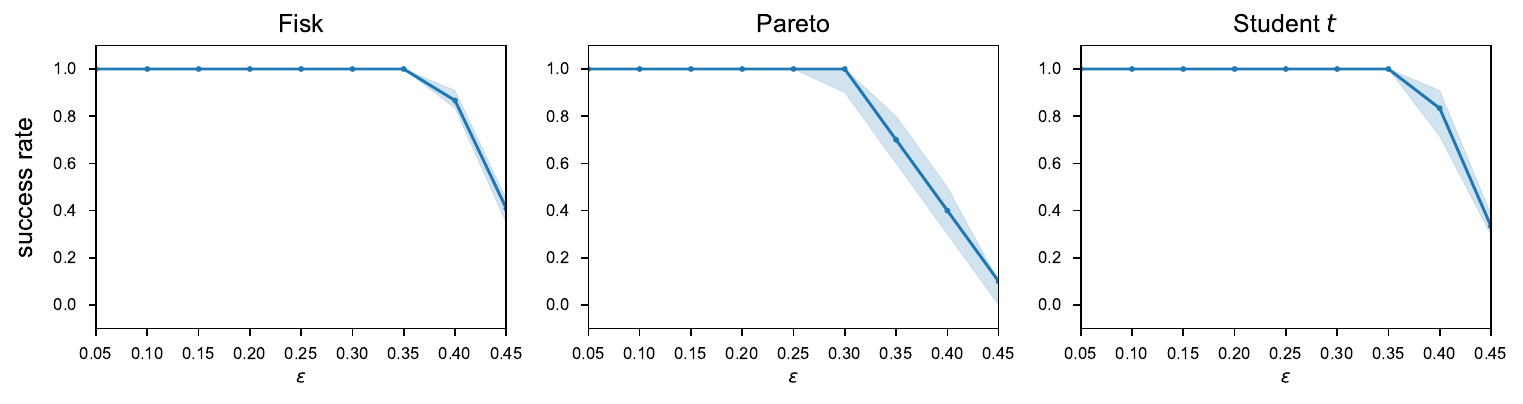}}\label{fig::success-rate}}\\
    \subfloat[Comparsion between Stage 1 and full algorithms with different sparsity $k$]{
    {\includegraphics[width=0.9\linewidth]{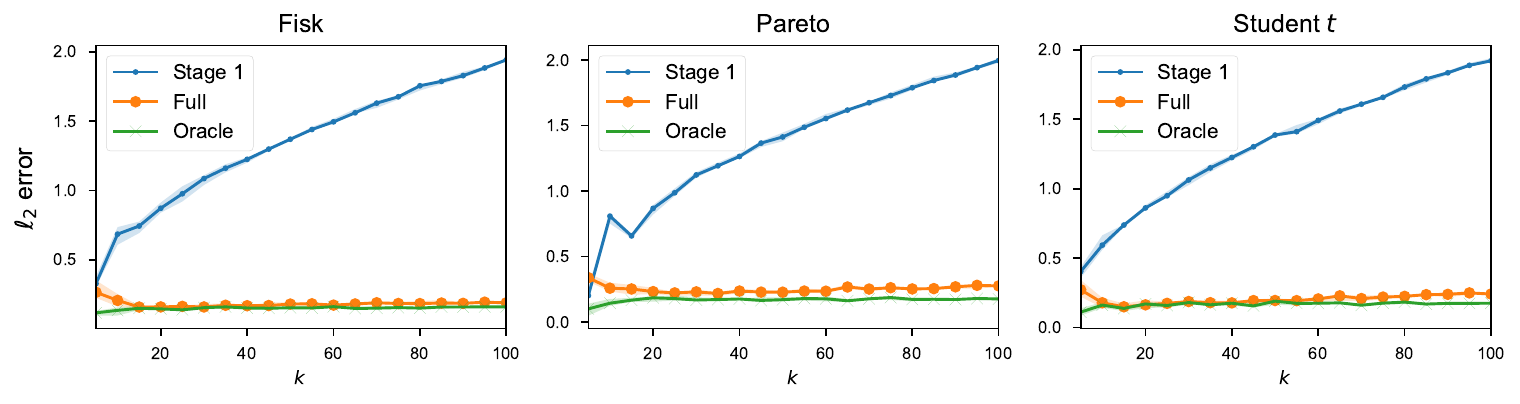}}\label{fig::three-loss-vs-k}}\\
    \subfloat[Performances of Stage 1 and full algorithms in infinite variance regime]{
    {\includegraphics[width=0.9\linewidth]{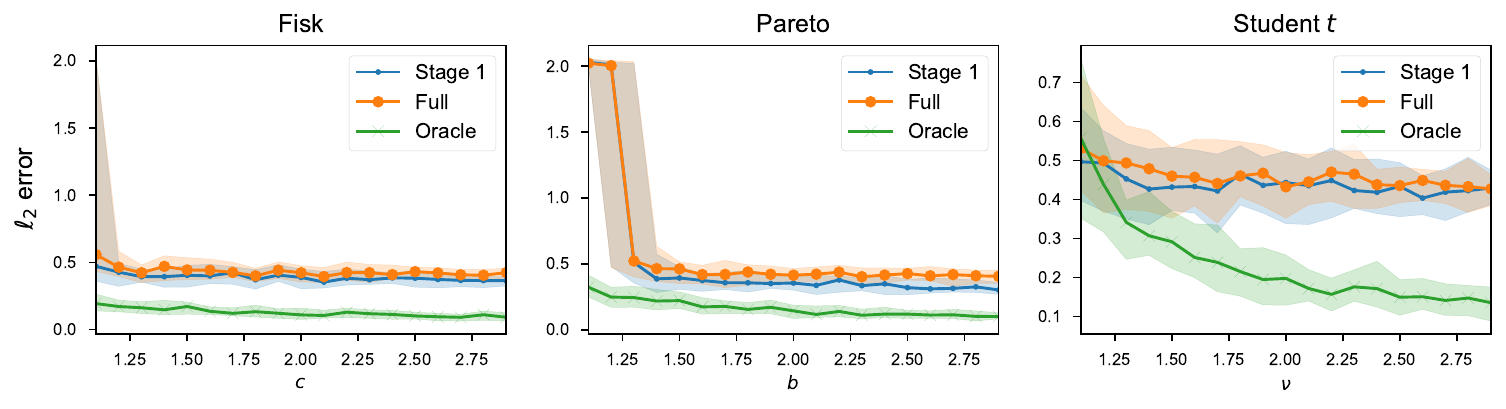}}\label{fig::three-loss-vs-param}}\\
    \end{centering}
    \caption{\footnotesize The data dimension is fixed at $d=100$. Unless otherwise specified, the corruption ratio is set to $\epsilon=0.1$ and the sparsity level to $k=4$. For the first two simulations, the distribution parameters $c, b, \nu$ are set to $3.1$ (see Appendix~\ref{sec::additional-simulations} for further details). The sample size is $n=600$ in the first and third simulations, while in the second simulation it is scaled as $n=100k$, varying proportionally with the sparsity level.}
    \label{fig::main-simulation}
\end{figure*}

\paragraph{Simulation setup.} All the experiments are conducted on a MacBook Pro 2021 with the Apple M1 Pro chip and a $16$GB unified memory. We pick three representative heavy-tailed probability distributions: Fisk, Pareto, and Student's $t$. To make a fair comparison, we fix the data dimension at $d=100$ and use the constant-bias noise model introduced in \citet{cheng2021outlier} to generate outliers. Unless otherwise stated, we set the corruption ratio at $\epsilon=0.1$ and the sparsity level at $k=4$. 
As for the algorithm in Stage 2, we utilize the filter-based algorithm \texttt{RME\_sp} introduced in \citet{diakonikolas2019outlier}. Furthermore, we compare our algorithms with the \textit{Oracle} estimator, which uses the coordinate-wise \textsf{MoM} on the clean data with an optimal choice of subgroup numbers. In all of our simulations, we set the number of iterations of SubGM to $200$, which is in line with our theoretical results.

\paragraph{Identification of top-$k$ elements.} In this experiment, we evaluate the success rate under varying corruption ratios $\epsilon$, while keeping all other parameters fixed. Our theoretical result (\Cref{thm::main}) indicates that provable identification is possible only when $\epsilon \lesssim \mu_{\min}^{\star 2}/\sigma^2$, suggesting that the success rate should deteriorate as $\epsilon$ increases. We define the recovered index set obtained by SubGM as $I$, and the true index set of the top-$k$ elements as $I_k$. The success rate is then measured as $|I \cap I_k| / |I \cup I_k|$. The results, presented in \Cref{fig::success-rate}, are averaged over $50$ independent trials for each setting. Notably, SubGM achieves exact recovery of the true index set $I$ even when up to $30\%$ of the samples are corrupted, highlighting the robustness and practical effectiveness of our method.

\paragraph{Comparison between Stage 1 and full algorithms.}
We evaluate the $\ell_2$-error of the Stage 1 and full algorithms across varying sparsity levels $k$. Our theoretical results predict a gap in $\ell_2$-error between the two algorithms---$\cO(\sigma\sqrt{k\epsilon})$ versus $\cO(\sigma\sqrt{\epsilon})$---when $k$ is sufficiently large. To minimize the influence of sample size, we set $n = 100k$, ensuring a sufficiently large number of samples. As shown in \Cref{fig::three-loss-vs-k}, the two algorithms perform comparably when $k$ is small. However, as $k$ increases, the $\ell_2$-error of Stage 1 grows sublinearly, while the full algorithm maintains a stable error level. These empirical findings are fully consistent with our theoretical predictions.

\paragraph{Infinite variance regime.}
In this experiment, we evaluate the performance of our algorithm in the infinite variance regime, fixing the sparsity level at $k=4$ and the sample size at $n=600$. When the distribution parameters $c, b, \nu$ fall within the interval $(1,2]$, the Fisk, Pareto, and Student’s $t$ distributions all exhibit infinite variance (see Appendix~\ref{sec::additional-simulations} for further details). As shown in \Cref{fig::three-loss-vs-param}, both Stage 1 and the full algorithm maintain strong performance in this setting, suggesting that our theoretical guarantees may extend to the infinite variance regime. Notably, Stage 1 consistently outperforms the full algorithm across all three distributions, implying that SubGM may possess greater robustness than existing estimators under infinite variance.

%% file: main_text/deferred_proofs.tex
\section{Proofs}
\label{sec::deferred-proofs}
The proofs of our main results are organized as follows. \Cref{sec::technical-lemmas} presents preliminary lemmas. \Cref{sec::proofofthoerem1} establishes the convergence guarantee of SubGM (\Cref{thm::main}), and \Cref{sec::proof-theorem2} provides the end-to-end guarantee of the full algorithm (\Cref{thm::main2}). \Cref{sec::proof-lower-bound} derives the information-theoretic lower bound (\Cref{lem::information-bound}), and Appendix~\ref{sec::mom} proves a formal variant of \Cref{prop::mom-1d} to establish the properties of the \textsf{MoM} estimator.

\subsection{Preliminaries}
\label{sec::technical-lemmas}
This section presents all the technical lemmas that will be used to prove our main results.
\begin{lemma}[Chebyshev's inequality \protect{\cite[Corollary~1.2.5]{vershynin2018high}}]
    \label{lem::Chebyshev}
Suppose that $X\sim \bP$ with $\Var(X)<\infty$. Then, for any $\delta > 0$, we have
\begin{equation}
    \pr\left(\left|X-\bE[X]\right|\geq \delta\right)\leq \frac{\Var(X)}{\delta^2}.
\end{equation}
\end{lemma}

\begin{lemma}[Hoeffding's inequality \protect{\cite[Theorem~2.2.6]{vershynin2018high}}]
    \label{lem::hoeffding}
    Let $X_1, \cdots, X_n$ be independent random variables such that $a_{i}\leq X_{i}\leq b_{i}$ almost surely. Then for all $\delta>0$, we have 
    \begin{equation}
        \pr\biggl(\sum_{i=1}^{n}X_i-\bE[X_i]\geq \delta\biggr)\leq \exp\biggl\{-\frac{2\delta^2}{\sum_{i=1}^{n}(b_i-a_i)^2}\biggr\}.
    \end{equation}
\end{lemma}
\begin{lemma}[Dvoretzky-Kiefer-Wolfowitz Inequality \citep{massart1990tight}]\label{lem_GC}
        Let $F(t) = \pr(X \leq t)$ be the CDF of a random variable $X$, and let $\hat F_n(\cdot)=\frac{1}{n}\sum_{i=1}^{n}\bI_{(-\infty, \cdot]}(X_i)$ be the empirical CDF based on $n$ i.i.d. samples $X_1,\dots, X_n\sim \bP$. We have 
        \begin{equation}
            \pr\left(\norm{\hat F_n-F}_{\infty}\geq t\right)\leq 2 e^{-2nt^2} \quad \text{for all }t \geq 0.
        \end{equation}
    \end{lemma}

\begin{lemma}
    \label{lem::2}
    Suppose $X_1, \cdots X_n\overset{i.i.d.}{\sim}\bP$. Then, with probability at least $1-\delta$ and for all $a\in \bR$, we have 
    \begin{equation}
        \biggl|\frac{1}{n}\sum_{i=1}^{n}\tildesign\left(X_i-a\right)-\bE\left[\tildesign\left(X-a\right)\right]\biggr|\leq \sqrt{\frac{2\log(2/\delta)}{n}}.
    \end{equation}
\end{lemma}
\proof{Proof}
    Note that $\tildesign(X_i-a)=1-\bI_{(-\infty, a]}(X_i)-\bI_{(-\infty, a)}(X_i)$ and $\bE\left[\tildesign\left(X-a\right)\right]=1-\pr(X\leq a)-\pr(X< a)$. Therefore, we have 
    \begin{equation}
        \begin{aligned}
            &\biggl|\frac{1}{n}\sum_{i=1}^{n}\tildesign\left(X_i-a\right)-\bE\left[\tildesign\left(X-a\right)\right]\biggr|\\
            &\leq \biggl|\frac{1}{n}\sum_{i=1}^{n}\bI_{(-\infty, a]}(X_i)-\pr(X\leq a)\biggr|+\sup_{b<a}\biggl|\frac{1}{n}\sum_{i=1}^{n}\bI_{(-\infty, b]}(X_i)-\pr(X\leq b)\biggr|\\
            &\leq 2\norm{\hat F_n-F}_{\infty}.
        \end{aligned}
    \end{equation}
    Upon setting $t=\sqrt{\frac{1}{2n}\log\left(\frac{2}{\delta}\right)}$ in the Dvoretzky-Kiefer-Wolfowitz Inequality (\Cref{lem_GC}), with probability at least $1-\delta$ and for all $a\in \bR$, we have
    \begin{equation}
        \begin{aligned}
            \biggl|\frac{1}{n}\sum_{i=1}^{n}\tildesign\left(X_i-a\right)-\bE\left[\tildesign\left(X-a\right)\right]\biggr|\leq 2\norm{\hat F_n-F}_{\infty}\leq \sqrt{\frac{2\log(2/\delta)}{n}}.\hfill\square
        \end{aligned}
    \end{equation}
\endproof

\begin{lemma}[Berry-Esseen bound \protect{\cite[Theorem~2.1.3]{vershynin2018high}}]
    \label{lem::berry-esseen}
    Suppose $X_1, \cdots X_n\overset{i.i.d.}{\sim}\bP$, where $\bP$ has zero mean and bounded third moment, i.e., $\mu=\bE[X]=0, \rho=\bE[|X|^3]<\infty$. Then, upon denoting $Z_n=(\sum_{i=1}^{n}X_i)/\sqrt{n \sigma^2}$ where $\sigma^2=\bE[X^2]$, we have 
    \begin{equation}
        \sup_{a\in \bR}\left|\pr\left(Z_n<a\right)-\Phi(a)\right|\leq \frac{0.5\rho}{\sigma^3\sqrt{n}}.
    \end{equation}
    Here $\Phi(\cdot)$ is the CDF of standard Gaussian distribution.
\end{lemma}

    \subsection{Proof of \Cref{thm::main}}
\label{sec::proofofthoerem1}

To prove this theorem, it is essential to first establish the uniform concentration of $\frac{1}{J}\sum_{j=1}^{J}\tildesign\left(\bar X_{j,i}+a\right)$ for all $a\geq 0$.
\begin{lemma}
    \label{lem::uniform-convergence}
    Suppose $X_1, \cdots X_n\overset{i.i.d.}{\sim}\bP$, where $\bP$ has zero mean, variance $\sigma^2$, and coordinate-wise third moment $\rho$. Moreover, suppose samples are generated according to the strong contamination model (\Cref{assumption::comtamination}) with corruption parameter $\epsilon$. Suppose $n\geq 20000\log(2d/\delta)/\epsilon$, $J= 100\lceil\epsilon n\rceil$, and $\rho\leq 0.005\sigma^3/\sqrt{\epsilon}$. Upon dividing the samples into $J$ equal subgroups $S_1, \cdots, S_J$ and denoting the empirical mean of each subgroup by $\bar X_j=\frac{1}{J}\sum_{k\in S_j}X_k$, with probability at least $1-\delta$, the following statements hold
    \begin{itemize}
        \item For all $a\geq 20\sigma\sqrt{\epsilon}$  and all $1\leq i\leq d$, we have: $\frac{3}{5}\leq \frac{1}{J}\sum_{j=1}^{J}\tildesign\left(\bar X_{j, i}+a\right)\leq 1.$
        \item For all $0\leq a\leq 0.001\sigma\sqrt{\epsilon}$ and all $1\leq i\leq d$, we have: $-0.08\leq \frac{1}{J}\sum_{j=1}^{J}\tildesign\left(\bar X_{j, i}+a\right)\leq 0.08.$
    \end{itemize}
\end{lemma}
\proof{Proof}
We prove the two cases separately. 

    \underline{\textit{Case 1}: $a\geq 20\sigma\sqrt{\epsilon}$}.
    We only need to prove the lower bound since the upper bound is trivial. We partition the index set of subgroups ${1, \ldots, J}$ into two disjoint subsets: $\cJ_{\text{clean}}$, containing all subgroups free of outliers, and $\cJ_{\text{outlier}}$, containing those with at least one outlier. Note that $|\cJ_{\text{outlier}}| \leq \lceil \epsilon n \rceil$. Therefore, we obtain
    \begin{equation}
        \frac{1}{J}\sum_{j=1}^{J}\tildesign\left(\bar X_{j, i}+a\right)\geq \frac{1}{J}\sum_{j\in \cJ_{\text{clean}}}\tildesign\left(\bar X_{j, i}+a\right)-\frac{\lceil\epsilon n\rceil}{J}.
    \end{equation}
    Next, applying \Cref{lem::2} and a union bound, we obtain that, with probability at least $1-\delta$ and for all $1\leq i\leq d$, 
    \begin{equation}
        \begin{aligned}
            \frac{1}{J}\sum_{j=1}^{J}\tildesign\left(\bar X_{j, i}+a\right)&\geq \frac{|\cJ_{\text{clean}}|}{J}\left(\bE\left[\tildesign\left(\bar X_{j, i}+a\right)\right]-\sqrt{\frac{2\log(2d/\delta)}{|\cJ_{\text{clean}}|}}\right)-\frac{\lceil\epsilon n\rceil}{J}\\
            &=\frac{|\cJ_{\text{clean}}|}{J}\left(1-2\pr\left(\bar X_{j, i}\leq -a\right)-\sqrt{\frac{2\log(2d/\delta)}{|\cJ_{\text{clean}}|}}\right)-\frac{\lceil\epsilon n\rceil}{J}.
        \end{aligned}
    \end{equation} 
    To proceed, one can write
    \begin{equation}
        \begin{aligned}
            \pr\left(\bar X_{j, i}\leq -a\right)&=\pr\biggl(\frac{\bar X_{j, i}}{\sqrt{\Var(X)/B}}\leq -\frac{a}{\sqrt{\Var(X)/B}}\biggr)\\
            &\stackrel{(a)}{\leq} \Phi\biggl(-\frac{a}{\sqrt{\Var(X)/B}}\biggr)+\frac{0.5\rho}{\sigma^3\sqrt{B}}\\
            &\stackrel{(b)}{\leq} \exp\biggl\{-\frac{Ba^2}{2\Var(X)}\biggr\}+\frac{0.5\rho}{\sigma^3\sqrt{B}}\\
            &\leq \exp\biggl\{-\frac{Ba^2}{2\sigma^2}\biggr\}+\frac{0.5\rho}{\sigma^3\sqrt{B}}.
        \end{aligned}
    \end{equation}
    Here, $(a)$ follows from the Berry-Esseen bound (\Cref{lem::berry-esseen}). In $(b)$, we use the concentration inequality for standard Gaussian distribution. Combining the above inequalities and recalling our choices of $J$, $B$, and $n$, we conclude that, with probability at least $1 - \delta$ and for all $1\leq i\leq d$,
    \begin{equation}
        \begin{aligned}
            \frac{1}{J}\sum_{j=1}^{J}\tildesign\left(\bar X_{j, i}+a\right)&\geq \frac{|\cJ_{\text{clean}}|}{J}\left(1-2\left(\exp\biggl\{-\frac{Ba^2}{2\sigma^2}\biggr\}+\frac{0.5\rho}{\sigma^3\sqrt{B}}\right)-\sqrt{\frac{2\log(2d/\delta)}{|\cJ_{\text{clean}}|}}\right)-\frac{\lceil\epsilon n\rceil}{J}\\
            &\geq 0.99\cdot\left(1-2\cdot\left(e^{-2}+0.025\right)-\sqrt{\frac{100}{99}\cdot\frac{1}{4500}}\right)-0.01\\
            &\geq \frac{3}{5}.
        \end{aligned}
    \end{equation}
    This completes the proof of the first statement.
    
    \underline{\textit{Case 2}: $0\leq a\leq 0.001\sigma\sqrt{\epsilon}$}. In this case, it suffices to provide an upper bound for $\left|\frac{1}{J}\sum_{j=1}^{J}\tildesign\left(\bar X_{j, i}+a\right)\right|$. Following a similar derivation as in \underline{\textit{Case 1}}, with probability at least $1-\delta$ and for all $1\leq i\leq d$, we have 

    \begin{equation}
        \begin{aligned}
            &\left|\frac{1}{J}\sum_{j=1}^{J}\tildesign\left(\bar X_{j, i}+a\right)\right|\\
            &\leq \frac{|\cJ_{\text{clean}}|}{J}\left(1-2\Phi\left(-\frac{a}{\sqrt{\Var(X)/B}}\right)+\frac{\rho}{\sigma^3\sqrt{B}}+\sqrt{\frac{2\log(2d/\delta)}{|\cJ_{\text{clean}}|}}\right)+\frac{\lceil\epsilon n\rceil}{J}\\
            &= \frac{|\cJ_{\text{clean}}|}{J}\left(2\Phi(0)-2\Phi\left(-\frac{a}{\sqrt{\Var(X)/B}}\right)+\frac{\rho}{\sigma^3\sqrt{B}}+\sqrt{\frac{2\log(2d/\delta)}{|\cJ_{\text{clean}}|}}\right)+\frac{\lceil\epsilon n\rceil}{J}\\
            &\stackrel{(a)}{\leq}\frac{1.98a}{\sqrt{\sigma^2/B}}+\frac{0.99\rho}{\sigma^3\sqrt{B}}+\sqrt{\frac{1.98\log(2d/\delta)}{J}}+\frac{\lceil\epsilon n\rceil}{J}\\
            &\le 1.98a\sqrt{B\epsilon}+0.05+\sqrt{\frac{1.98}{9000}}+0.01\\
            &\leq 0.08.
        \end{aligned}
    \end{equation}
    Here, in $(a)$, we use the anti-concentration for the standard Gaussian distribution. This completes the proof of the second statement. 
\endproof
We are now ready to present the proof of \Cref{thm::main}. To this goal, we first present a more precise version of its statement.

\begin{theorem}[Convergence guarantee for SubGM]
    \label{thm::main-v2}
    Let $\bP$ be a distribution on $\bR^d$ with an unknown $k$-sparse mean $\vmu^{\star}$, unknown covariance matrix $\mSigma \preceq \sigma^2\mI$, and unknown coordinate-wise third moment satisfying $\bE[|X_i-\mu_i^{\star}|^3]\leq 0.005\sigma^3/\sqrt{\epsilon}, \forall 1\leq i\leq d$. Suppose a sample set of size $n\geq 20000\log(2d/\delta)/\epsilon$ is collected according to the strong contamination model (\Cref{assumption::comtamination}) with corruption parameter $\epsilon$. Upon setting the stepsize $\eta\leq \sigma\sqrt{\epsilon}/\mu_{\max}^{\star}$ and the initialization scale $0<\alpha\leq 0.001\sigma\sqrt{\epsilon/d}\wedge \mu_{\max}^{\star -5}$ in \Cref{alg:main}, with a probability of at least $1-\delta$, the following statements hold for any iteration $\frac{2}{\eta}\log(1/\alpha)\leq T\leq \frac{6}{\eta}\log(1/\alpha)$:
\begin{itemize}
    \item \textbf{Near optimal $\ell_2$-error.} The $\ell_2$-error is upper-bounded by
    \begin{equation}
        \norm{\hat\vmu(T)-\vmu^{\star}}\leq 31\sigma\sqrt{k\epsilon}.
    \end{equation}
    \item \textbf{Coordinate-wise error bound.} We obtain 
    \begin{equation}
        \begin{aligned}
            |\hat{\mu}_i(T)-\mu_i^\star|&\leq 30\sigma\sqrt{\epsilon}, \quad &\text{where }\mu_i^{\star}\neq 0,\\
            |\hat{\mu}_i(T)|&\leq \alpha, \quad &\text{where }\mu_i^{\star}= 0.
        \end{aligned}
    \end{equation}
\end{itemize}
\end{theorem}
Before proceeding to the proof, we note that second statement of \Cref{thm::main-v2} together with the assumption $\epsilon\leq \mu_{\min}^{\star 2}/(961\sigma^2)$ readily implies $\hat{\mu}_i(T)\geq \sigma\sqrt{\epsilon}$ for every $i$ such that $\mu^\star_i\not=0$, leading to the second statement of \Cref{thm::main}.
\proof{Proof of \Cref{thm::main-v2}}
Let us define $\cI_{\text{residual}} = \{i: \mu_i^\star = 0\}$ and $\cI_{\text{signal}} = \{i: \mu_i^\star \not= 0\}$.
We analyze coordinate-wise dynamics $\hat\mu_i(t):=u_i^2(t)-v_i^2(t)$ separately for signals $\cI_{\text{signal}}$ and residuals $\cI_{\text{residual}}$.
\paragraph{Signal dynamics.} Without loss of generality, we assume that $\mu_i^{\star} > 0$. Let us first revisit the update rule for SubGM:
\begin{equation}
    \begin{aligned}
        u_i(t+1)&=\biggl(1+\eta \frac{1}{J}\sum_{j=1}^{J}\tildesign\left(\bar X_{j,i}-\hat\mu_i(t)\right)\biggr)u_i(t),\\
        v_i(t+1)&=\biggl(1-\eta \frac{1}{J}\sum_{j=1}^{J}\tildesign\left(\bar X_{j,i}-\hat\mu_i(t)\right)\biggr)v_i(t).
    \end{aligned}
    \label{eq::21}
\end{equation}
We further divide our analysis into two cases depending on the magnitude of $\left|\mu_i^{\star}\right|$.

\underline{\textit{Case 1}: $\mu_i^{\star}\geq 20\sigma\sqrt{\epsilon}$}. We define $T_i=\left\{\min t: \mu_i^{\star}-\hat\mu_i(t)< 20\sigma\sqrt{\epsilon}\right\}$. Hence, for all $0\leq t\leq T_i$, the first statement of \Cref{lem::uniform-convergence} can be invoked to show
\begin{equation}
    0.6\leq\frac{1}{J}\sum_{j=1}^{J}\tildesign\left(\bar X_{j, i}-\hat\mu_i(t)\right)\leq 1.
\end{equation}
By incorporating this into \Cref{eq::21}, we obtain 
\begin{align}
    u_i^2(t+1)&\geq \left(1+0.6\eta\right)^2u_i^2(t)\geq (1+1.2\eta)u_i^2(t),\\
    v_i^2(t+1)&\leq \left(1-0.6\eta\right)^2v_i^2(t)\leq v_i^2(t).
\end{align}
Notice that $v_i(0)=\alpha$ at the initialization. We find that $v_i^2(t) \leq \alpha^2, \forall 0\leq t \leq T_i$, which remains adequately small throughout the trajectory. Next, we examine the dynamics of $u_i^2(t)$. Taking into account that $u_i^2(0) = \alpha^2$ and $u_i^2(t) \geq (1+1.2\eta)^t u_i^2(0)$, we have that within $T_i \leq \frac{5}{3\eta}\log\left(\frac{|\mu_i^{\star}|}{\alpha}\right)$ iterations, the following holds
\begin{equation}
    u_i^2(T_i)\geq \alpha^2 (1+1.2\eta)^{T_i}\geq \mu_i^\star - 10\sigma\sqrt{\epsilon}.
\end{equation}
This implies
\begin{equation}
    \hat\mu_i(T_i)=u_i^2(T_i)-v_i^2(T_i)\geq \mu_i^\star - 10\sigma\sqrt{\epsilon}-\alpha^2\geq \mu_i^{\star}-20\sigma \sqrt{\epsilon}.
\end{equation}
Next, we show that $\left|\mu_i^{\star}-\hat\mu_i(T_i^{\star})\right|\leq 20\sigma \sqrt{\epsilon}$. To this goal, when $t< T_i$, we provide an upper bound on the difference between two consecutive iterations as follows
\begin{equation}
    \begin{aligned}
        |\hat\mu_i(t+1)-\hat\mu_i(t)|&\leq \left|u^2_i(t+1)-u^2_i(t)\right|+\left|v^2_i(t+1)-v^2_i(t)\right|\\
        &\stackrel{(a)}{\leq} \Bigg|\biggl(1+\eta \frac{1}{J}\sum_{j=1}^{J}\tildesign\left(\bar X_{j,i}-\hat\mu_i(t)\right)\biggr)^2-1\Bigg|\cdot u_i^2(t) + \alpha^2\\
        &\stackrel{(b)}{\leq} \left((1+\eta)^2-1\right) u_i^2(t) + \alpha^2\\
        &\stackrel{(c)}{\leq} 3\eta \mu_i^\star + 2\alpha^2\\
        &\stackrel{(d)}{\leq} 4\sigma\sqrt{\epsilon}.
    \end{aligned}
    \label{eq::39}
\end{equation}
Here in $(a)$, we use the fact that $\left|v^2_i(t+1)-v^2_i(t)\right|\leq \max\{v^2_i(t+1),v^2_i(t)\}\leq \alpha$. In $(b)$, we use the estimate $0.6\leq\frac{1}{J}\sum_{j=1}^{J}\tildesign\left(\bar X_{j, i}-\hat\mu_i(t)\right)\leq 1$. In $(c)$, we use the fact that $u_i^2(t)\leq \hat{\mu}_i(t)+v_i^2(t)\leq \mu_i^\star+\alpha^2$. Lastly, in $(d)$, we use the condition that $\eta \leq \frac{\sigma\sqrt{\epsilon}}{\mu_{\max}^{\star}}$ and $\alpha^2\leq 0.5\sigma\sqrt{\epsilon}$.
Hence, we have
\begin{equation}
    \hat\mu_i(T_i)-\mu_i^\star\leq \underbrace{\hat\mu_i(T_i-1)-\mu_i^\star}_{\leq 0 \text{ by definition of }T_i} + \left(\hat\mu_i(T_i)-\hat\mu_i(T_i-1)\right)\leq 4\sigma \sqrt{\epsilon}.
\end{equation}
Combining with the fact that $\hat\mu_i(T_i)- \mu_i^{\star}\geq-20\sigma \sqrt{\epsilon}$, we derive that $\left|\mu_i^{\star}-\hat\mu_i(T_i^{\star})\right|\leq 20\sigma \sqrt{\epsilon}$.

We will now demonstrate that for any $t \geq T_i^{\star}$, the condition $\left|\mu_i^{\star}-\hat\mu_i(t)\right| \leq 30\sigma\sqrt{\epsilon}$ always holds. Using the fact $n\geq 20000\log(2d/\delta)/\epsilon$ and \Cref{thm::mom-high-dim} (in the appendix), we have $|\hat{\mu}_{\textsf{MoM}}-\mu_i^{\star}| \leq 5\sigma\sqrt{\epsilon}$. Then, the triangle inequality implies
\begin{equation}
    \left|\mu_i^{\star}-\hat\mu_i(t)\right|\leq \left|\hat{\mu}_{\textsf{MoM}}-\mu_i^{\star}\right|+\left|\hat{\mu}_{\textsf{MoM}}-\hat\mu_i(t)\right|.
\end{equation}
Therefore, it suffices to show that $\left|\hat{\mu}_{\textsf{MoM}}-\hat\mu_i(t)\right|\leq 25\sigma\sqrt{\epsilon}$ for every $t\geq T_i^\star$. To this goal, we use induction on $t$. For $t=T_i^\star$, we have
\begin{equation}
    \left|\hat{\mu}_{\textsf{MoM}}-\hat\mu_i(T_i^{\star})\right|\leq \left|\hat{\mu}_{\textsf{MoM}}-\mu_i^{\star}\right|+\left|\mu_i^{\star}-\hat\mu_i(T_i^{\star})\right|\leq 25\sigma\sqrt{\epsilon}.
\end{equation}
Now, let us assume that at time $t\geq T_i^\star$, $\left|\hat{\mu}_{\textsf{MoM}}-\hat\mu_i(t)\right| \leq 25\sigma\sqrt{\epsilon}$. Without loss of generality, we assume $\hat\mu_i(t) \leq \hat{\mu}_{\textsf{MoM}}$. Based on the definition of the \textsf{MoM} estimator, we have
\begin{equation}
    \sum_{j=1}^{J}\tildesign\left(\bar X_{j, i}-\hat{\mu}_{\textsf{MoM}}\right)=0\implies\sum_{j=1}^{J}\tildesign\left(\bar X_{j, i}-\hat\mu_i(t)\right)\geq 0.
\end{equation}
Let $\beta_i(t) = \frac{1}{J}\sum_{j=1}^{J}\tildesign\left(\bar X_{j, i}-\hat\mu_i(t)\right)$. With this notation, we can derive the following inequality
\begin{equation}
    \begin{aligned}
        \hat\mu_i(t+1)-\hat\mu_i(t)&=(2\eta\beta_i(t)+\eta^2\beta_i^2(t))u_i^2(t)+(2\eta\beta_i(t)-\eta^2\beta_i^2(t))v_i^2(t)\geq 0,
    \end{aligned}
\end{equation}
where in the last inequality, we use the fact that $u_i^2(t)\geq v_i^2(t)$.
On the other hand, following exactly the same argument in \Cref{eq::39}, we have
\begin{equation}
    \begin{aligned}
        \hat\mu_i(t+1)-\hat\mu_i(t)\leq 4\sigma\sqrt{\epsilon}.
    \end{aligned}
\end{equation}
By combining the above two inequalities, we establish that $\left|\hat{\mu}_{\textsf{MoM}}-\hat\mu_i(t+1)\right| \leq 25\sigma\sqrt{\epsilon}$. This completes the proof of induction.

\underline{\textit{Case 2}: $|\mu_i^{\star}|\leq 20\sigma\sqrt{\epsilon}$}. Since $\hat{\mu}_i(0)=0$, at iteration $t=0$ we already have $\left|\mu_i^{\star}-\hat\mu_i(t)\right| \leq 20\sigma\sqrt{\epsilon}$. Consequently, the analysis reduces to the last phase of \underline{\textit{Case 1}}, from which we can conclude $\left|\mu_i^{\star}-\hat\mu_i(t)\right| \leq 30\sigma\sqrt{\epsilon}$ for all $t\geq 0$.

\paragraph{Residual dynamics.} In this case, we employ induction on $t$ to demonstrate that $|u_i^2(t)-v^2_i(t)| \leq \alpha$ for all $0 \leq t \leq T$. For the base case, this relationship is valid as $u_i^2(0)-v^2_i(0) = 0$. Assuming that this relation holds at time $t$, we can refer to \Cref{lem::uniform-convergence} and deduce
\begin{equation}
    -0.08\leq \frac{1}{J}\sum_{j=1}^{J}\tildesign\left(\bar X_{j, i}\right)\leq 0.08.
\end{equation}
Hence, we have 
\begin{equation}
    \begin{aligned}
        u_i^2(t+1)&\leq (1+0.08\eta)^2u_i^2(t)\leq \left(1+\eta/6\right)u_i^2(t),\\
        v_i^2(t+1)&\leq (1+0.08\eta)^2v_i^2(t)\leq \left(1+\eta/6\right)v_i^2(t).
    \end{aligned}
\end{equation}
Therefore, for all $t\leq \frac{6}{\eta}\log\left(\frac{1}{\alpha}\right)$, we obtain 
\begin{equation}
    \left|u_i^2(t)-v_i^2(t)\right|\leq \max\left\{u_i^2(t), v_i^2(t)\right\}\leq \alpha^2(1+\eta/6)^t\leq \alpha.
\end{equation}

\paragraph{Putting everything together.} Finally, since we set $\alpha \leq \frac{0.001}{\sqrt{d}}\sigma\sqrt{\epsilon}\wedge \mu_{\max}^{\star -5}$, for any $\frac{2}{\eta}\log\left(\frac{1}{\alpha}\right)\leq T\leq \frac{6}{\eta}\log\left(\frac{1}{\alpha}\right)$, we have 
\begin{equation}
    \begin{aligned}
        \norm{\hat\vmu(T)-\vmu^{\star}}_2&\leq \sqrt{k}\cdot 30\sigma\sqrt{\epsilon} + \sqrt{d}\alpha\leq 31\sigma\sqrt{k\epsilon}.
    \end{aligned}
\end{equation}
This completes the proof.
\endproof

\subsection{Proof of \Cref{thm::main2}}
\label{sec::proof-theorem2}
The proof follows by combining \Cref{thm::main} and \Cref{prop::k-dimensional}.
First, for the data distribution and corruption model considered in \Cref{thm::main}, once we set the sample size $n\gtrsim \log(d/\delta)/\epsilon$, then with probability at least $1-\delta/2$, we can successfully determine the location of the top-$k$ nonzero elements. For short, we represent the indices of these top-$k$ elements as $I_k$. Following the successful determination of these indices, we can then narrow our focus to a $k$-dimensional subproblem on the dataset $S_k:=\{[X_i]_{I_k}:X_i\in S\}$ with the mean $[\mu^{\star}]_{I_k}$. We can then apply \Cref{prop::k-dimensional} to this reduced dataset. Specifically, once the sample size satisfies $n\gtrsim (k+\log(d/\delta))/\epsilon$, there exists an estimator such that with probability at least $1-\delta/2$, it can output a $\hat{\mu}$ satisfying $\norm{\hat{\mu}-[\mu^{\star}]_{I_k}}\lesssim \sigma\sqrt{\epsilon}$. 

Combining these two steps via a simple union bound, we know that with a probability of at least $1-\delta$, our two-stage estimator $\hat{\mu}$ satisfies $\norm{\hat{\mu}-\mu^{\star}}\lesssim \sigma\sqrt{\epsilon}$. This concludes the proof. $\hfill\square$

\subsection{Proof of \Cref{lem::information-bound}}
\label{sec::proof-lower-bound}
Consider two probability distributions $\bP_1$ and $\bP_2$, where $\bP_2 = (1-\epsilon)\bP_1 + \epsilon \bQ$ for some distribution $\bQ$. Suppose we draw $n$ i.i.d. samples from $\bP_1$. Under the strong contamination model (\Cref{assumption::comtamination}) with parameter $\epsilon$, this same set of samples can be equivalently viewed as $\epsilon$-corrupted samples from $\bP_2$. Consequently, no algorithm can distinguish between the two cases (see \citet{li2019lecture2} for details).

Therefore, it suffices to construct two probability distributions satisfying the conditions in \Cref{lem::information-bound}. Without loss of generality, we focus on the one-dimensional case, since additional coordinates can be set identically. We require two distributions $\bP_1, \bP_2$ such that:
\begin{itemize}
\item Both distributions have variance at most $\sigma^2$ and third central moment at most $\sigma^3/\sqrt{\epsilon}$;
\item $\bP_2$ can be written as $(1-\epsilon)\bP_1 + \epsilon \bQ$ for some distribution $\bQ$;
\item Their means $\mu_1, \mu_2$ satisfy $|\mu_1 - \mu_2| \geq \sigma\sqrt{\epsilon}$.
\end{itemize}

Following \citet{li2019lecture2}, we construct $\bP_1$ as the point mass at $0$, and let $\bP_2 = (1-\epsilon)\bP_1 + \epsilon \bQ$, where $\bQ$ is the point mass at $\sigma/\sqrt{\epsilon}$. It is straightforward to verify that $\bP_1$ and $\bP_2$ satisfy all three conditions, completing the proof. \hfill$\square$

%% file: main_text/conclusion.tex
\section{Conclusion}
\label{sec::conclusion}

Many estimation tasks in statistics become notoriously difficult in the robust setting when certain assumptions on the data are lifted. For instance, almost all statistically optimal robust mean estimators suffer from overwhelmingly high computational costs. 
While classical results in robust statistics have shed light on the statistical limits of robust estimation, its computational aspects have mostly remained elusive.
In this work, we aim to bridge this gap by presenting the \textit{first} computationally efficient and statistically optimal method for robust sparse mean estimation, thereby overcoming a conjectured computational-statistical barrier under moderate conditions. 

%% file: main_text/acknowledge.tex
\section*{Acknowledgements}
We thank Jikai Hou for insightful discussions. SF is supported, in part, by NSF Award DMS-2152776, ONR Award N00014-22-1-2127, and MICDE Catalyst Grant.

%% file: appendix/additional_simulation.tex
\section{\textsf{MoM} Estimator under Strong Contamination Model}
\label{sec::mom}

In this section, we prove the key properties of the $1$-dimensional and high-dimensional \textsf{MoM} estimators under the strong contamination model (\Cref{assumption::comtamination}). The following is a more precise statement of \Cref{prop::mom-1d}, which is adapted from Fact 2.1. in \citet{diakonikolas2022outlier}. As the complete proof does not appear in the original source, we include it here for completeness.
\begin{proposition}[One-dimensional \textsf{MoM} estimator]
    \label{thm::mom-1d}
    Consider a corruption parameter $\epsilon$, failure probability $\delta$, and a set $S$ of $n$ many $\epsilon$-corrupted samples from a distribution $\bP$ with mean $\mu^{\star}$ and variance $\bE[(X-\mu^{\star})^2]\leq \sigma^2$. Then, with probability at least $1-\delta$, the \textsf{MoM} estimator $\hat{\mu}_{\textsf{MoM}}$ satisfies $|\hat{\mu}_{\textsf{MoM}}-\mu^{\star}|\leq \sigma\left(4\sqrt{2}\left(\sqrt{\epsilon}+\sqrt{1/n}\right)+16\sqrt{\log(1/\delta)/n}\right)$.
\end{proposition}
    \proof{Proof}
        We partition the index set of the subgroups $\{1, \cdots, J\}$ into two parts: $\cJ_{\text{clean}}$ and $\cJ_{\text{outlier}}$. Here $\cJ_{\text{clean}}$ comprises all the subgroups without outliers, and $\cJ_{\text{outlier}}$ consists of subgroups containing at least one outlier. According to our strong contamination model, we have $|\cJ_{\text{outlier}}|\leq \lceil \epsilon n\rceil$. Subsequently, we observe that 
        \begin{equation}
            \{|\hat{\mu}_{\textsf{MoM}}-\mu^{\star}|\geq \xi\}\subseteq \biggl\{\sum_{j\in \cJ_{\text{clean}}}\bI(|\bar X_{j}-\mu^{\star}|\geq \xi)\geq \frac{J}{2}-\lceil \epsilon n\rceil\biggr\}.
        \end{equation}
        Here, $\bar X_{j}=\frac{1}{B}\sum_{i\in S_j}X_i$, where $B=n/J$ is the size of each subgroup and $S_j$ is the subgroup $j$. For simplicity, let us denote $Z_{j}=\bI(|\bar X_{j}-\mu^{\star}|\geq \xi)$ and $p_{\xi}=\pr\left(|\bar X_{j}-\mu^{\star}|\geq \xi\right)$. Then, the above inclusion implies 
        \begin{equation}
            \begin{aligned}
                \pr\left(|\hat{\mu}_{\textsf{MoM}}-\mu^{\star}|\geq \xi\right)&\leq \pr\biggl(\sum_{j\in \cJ_{\text{clean}}}Z_{j}\geq \frac{J}{2}-\lceil \epsilon n\rceil\biggr)\\
                &=\pr\biggl(\frac{1}{|\cJ_{\text{clean}}|}\sum_{j\in \cJ_{\text{clean}}}\left(Z_{j}-\bE[Z_{j}]\right)\geq \frac{J/2-\lceil \epsilon n\rceil}{|\cJ_{\text{clean}}|}-p_{\xi}\biggr).
            \end{aligned}
        \end{equation}
        Since $Z_{j}$ is bounded, we can apply Hoeffding's inequality (\Cref{lem::hoeffding}) to obtain
        \begin{equation}
            \pr\left(|\hat{\mu}_{\textsf{MoM}}-\mu^{\star}|\geq \xi\right)\leq \exp\biggl\{-2|\cJ_{\text{clean}}|\left(\frac{J/2-\lceil \epsilon n\rceil}{|\cJ_{\text{clean}}|}-p_{\xi}\right)^2\biggr\}.
            \label{eq::9}
        \end{equation}
        Moreover, we can use Chebyshev's inequality (\Cref{lem::Chebyshev}) to establish an upper bound for $p_\xi$:
        \begin{equation}
            p_{\xi}=\pr\left(|\bar X_{j}-\mu^{\star}|\geq \xi\right)\leq \frac{\sigma^2}{B\xi^2}=\frac{J\sigma^2}{n\xi^2}.
        \end{equation}
        Upon defining $J=4\lceil \epsilon n\rceil+32\log(1/\delta)$ and $\xi=\sigma\left(4\sqrt{2}\left(\sqrt{\epsilon}+\sqrt{1/n}\right)+16\sqrt{\log(1/\delta)/n}\right)$, we have the following estimates
        \begin{equation}
            \begin{aligned}
                |\cJ_{\text{clean}}|&\geq J -\lceil \epsilon n\rceil\geq 32\log(1/\delta);\\
                \frac{J/2-\lceil \epsilon n\rceil}{|\cJ_{\text{clean}}|}&\geq \frac{J/2-\lceil \epsilon n\rceil}{J}\geq \frac{1}{4};\\
                p_\xi&\leq \frac{J\sigma^2}{n\xi^2}\leq \frac{1}{8}.
            \end{aligned} 
        \end{equation}
        Combining these bounds with~\Cref{eq::9}, we obtain
        \begin{equation}
            \begin{aligned}
                &\pr\left(|\hat{\mu}_{\textsf{MoM}}-\mu^{\star}|\geq \sigma\left(4\sqrt{2}\left(\sqrt{\epsilon}+\sqrt{1/n}\right)+16\sqrt{\log(1/\delta)/n}\right)\right)\\&\leq \exp\biggl\{-2\cdot 32\log(1/\delta) \cdot \left(\frac{1}{4}-\frac{1}{8}\right)^2\biggr\}=\delta.
            \end{aligned}
        \end{equation}
        This completes the proof. $\hfill\square$
    \endproof

    Directly applying \textsf{MoM} estimator to each coordinate of a $d$-dimensional dataset leads to the following proposition.

    \begin{theorem}[High dimensional coordinate-wise \textsf{MoM} estimator]
        \label{thm::mom-high-dim}
        Consider a corruption parameter $\epsilon$, failure probability $\delta$, and a set $S$ of $n$ many $\epsilon$-corrupted samples from a distribution $\bP$ with mean $\mu^{\star}$ and coordinate-wise variance $\bE[(X-\mu^{\star})^2]\leq \sigma^2, \forall 1\leq i\leq d$. Then, with probability at least $1-\delta$, the coordinate-wise \textsf{MoM} estimator $\hat{\mu}_{\textsf{MoM}}$ satisfies $\norm{\hat{\vmu}_{\textsf{MoM}}-\vmu^{\star}}_{\infty}\leq \sigma\left(4\sqrt{2}\left(\sqrt{\epsilon}+\sqrt{1/n}\right)+16\sqrt{\log(d/\delta)/n}\right)$ and $\norm{\hat{\vmu}_{\textsf{MoM}}-\vmu^{\star}}_{2}\leq \sigma \sqrt{d}\left(4\sqrt{2}\left(\sqrt{\epsilon}+\sqrt{1/n}\right)+16\sqrt{\log(d/\delta)/n}\right)$.
    \end{theorem}
    \proof{Proof}
        The proof follows directly from \Cref{thm::mom-1d} and a simple union bound. $\hfill\square$
    \endproof

\section{Additional Simulations}
\label{sec::additional-simulations}

\subsection{Experimental Details}
We run our simulations on three heavy-tailed distributions: Fisk, Pareto, and Student's $t$ distributions. In each case, we apply a symmetrization trick to make the density function symmetric around zero. The density function of the Fisk distribution with parameter $c$ is expressed as follows:
\begin{equation}
    f(x; c)=\frac{c|x|^{c-1}}{2(1+|x|^c)^2}\quad \text{for }x\in \bR, c>0.
\end{equation}
The density function of the Pareto distribution with parameters $b$ is
\begin{equation}
    f(x; b)=\{ \begin{array}{cc}
         \frac{b}{2|x|^{b+1}} & \mbox{for}\quad |x|\geq 1, \\ 0  & \mbox{for}\quad |x|<1.
                \end{array}, \quad \text{for }x\in \bR, b>0.
\end{equation}
Lastly, the density function for student $t$-distribution is 
\begin{equation}
    f(x;\nu)=\frac{\Gamma\left(\frac{\nu+1}{2}\right)}{\sqrt{\nu\pi}\Gamma(\nu/2)}\left(1+\frac{x^2}{\nu}\right)^{-(\nu+1)/2}\quad \text{for }x\in \bR, \nu>0.
\end{equation}
Here $\Gamma$ is the gamma function.
In all three distributions described above, the parameters $c, b, \nu$ correspondingly denote the existence of the $c, b, \nu$-th moment. For instance, when $c, b, \nu$ fall within the range of $(1, 2]$, the variances are infinite. Regarding the outliers, we generate them via the constant-bias noise model as introduced in \citet{cheng2021outlier}.

Furthermore, unless stated otherwise, all simulations are conducted with the following predefined settings: data dimension $d$ is set to $100$, sparsity level $k$ is set to $4$ with nonzero elements being $[10,-5,-4,2]$, sample size $m$ is set to 600, and the corruption ratio $\epsilon$ is set at $10\%$.
As for our algorithm, we set the number of subgroups to be $J = 1.5\lceil\epsilon n\rceil + 150$. {Note that, compared to the theoretical choice of $J = 100\lceil\epsilon n\rceil$ in \Cref{alg:main}, we choose a smaller $J$ to make our algorithm work for a larger corruption ratio $\epsilon$ in practice.} Moreover, in SubGM, we set the initialization scale $\alpha=10^{-5}$ and the step-size $\eta=0.05$. 

We select \texttt{sparse\_GD} \citep{cheng2021outlier} and \texttt{sparse\_filter} \citep{diakonikolas2019outlier} as our benchmark algorithms. We note that these algorithms \textit{do not} come with theoretical assurances in the heavy-tailed setting. Nonetheless, we have empirically found that these two algorithms surpass others in performance, even in the heavy-tailed setting. We also highlight that the polynomial-time algorithms that come equipped with theoretical guarantees for heavy-tailed setting \citep{diakonikolas2022outlier, diakonikolas2022robust} are impractical since they rely on time-consuming methods such as sum-of-squares and ellipsoid methods. 

We employ both \texttt{sparse\_GD} and \texttt{sparse\_filter} in the second stage of our algorithm, setting the sparsity parameter to $k = |I|$, where $I$ is the index set identified in the first stage. In total, we evaluate six estimators: \texttt{oracle} (which removes all outliers), \texttt{sparse\_GD}, \texttt{sparse\_filter}, \texttt{stage\_1}, \texttt{full\_GD} (our algorithm with \texttt{sparse\_GD} in the second stage), and \texttt{full\_filter} (our algorithm with \texttt{sparse\_filter} in the second stage). In \texttt{stage\_1}, we run SubGM for $T=600$ iterations, whereas in \texttt{full\_GD} and \texttt{full\_filter}, we reduce the iteration count to $T=200$ to lower computational cost.

\subsection{Sensitivity to Prior Knowledge of $k$}
We underscore the fact that prior algorithms necessitate a prior knowledge of the exact sparsity level $k$. In contrast, our approach can identify the sparsity level automatically. For this simulation, we assign a true sparsity level of $k=10$ with nonzero components $[2,2,2,2,2,-2,-2,-2,-2,-2]$ and assess the performance of the benchmark algorithms, namely \texttt{sparse\_GD} and \texttt{sparse\_filter}, while varying the input $k'$, which is an upper bound of $k$, within the range of [10, 40]. As illustrated in \Cref{fig::prior-knowledge}, the performance of these benchmark algorithms is highly sensitive to the choice of $k'$ across all examined distributions. Their performances further destabilize when the underlying distributions start to exhibit heavier tails. In contrast, our algorithm automatically recognizes the sparsity pattern across all scenarios. For all subsequent simulations, we provide the benchmark algorithms with the true sparsity level $k$ to ensure a fair comparison.
\begin{figure*}[h]
    \centering
    \includegraphics[width=\textwidth]{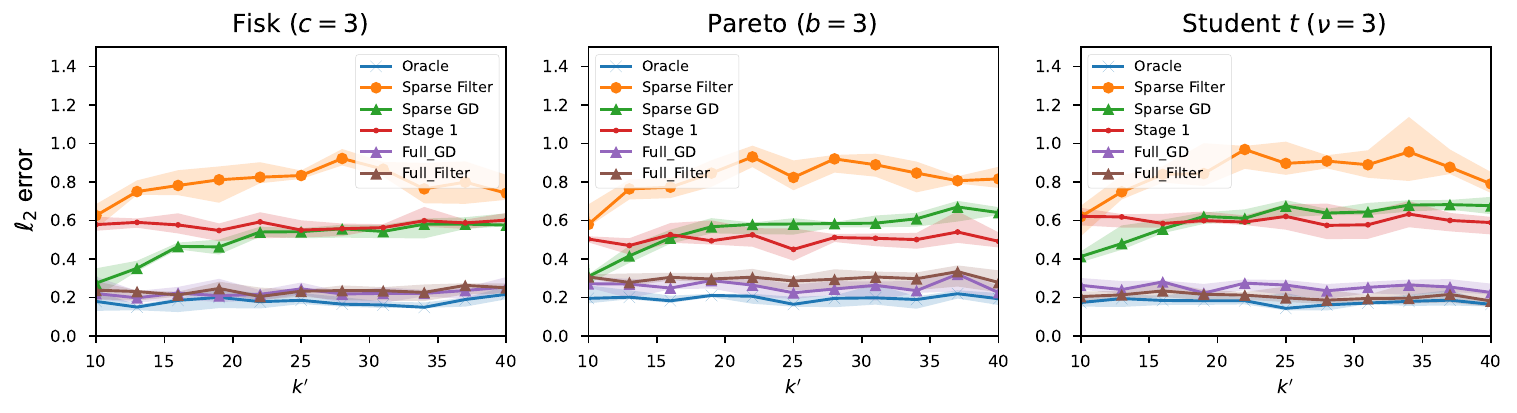}\\
    \includegraphics[width=\textwidth]{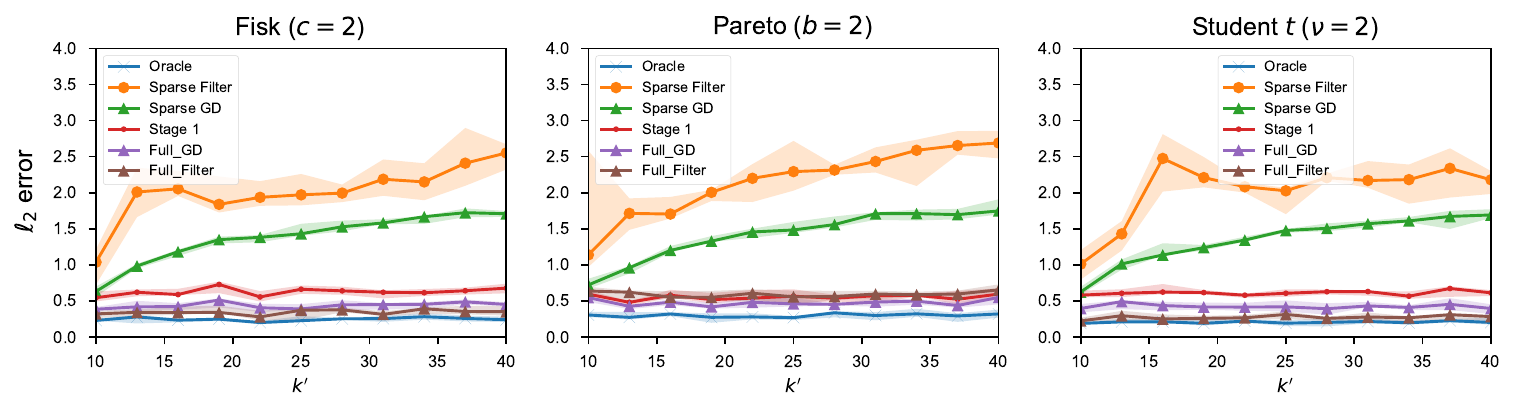}
    \caption{Comparison among different algorithms with varying input $k'$, where $k=10$ and $k'\geq k$ is an upper bound of $k$. The second row corresponds to distributions with infinite variance.}
    \label{fig::prior-knowledge}
\end{figure*}

\subsection{Performance with Different $k$}
In this simulation, we evaluate the performance of various algorithms under different sparsity levels $k$. We set all nonzero entries of $\mu^{\star}$ to $2$. As shown in the first row of \Cref{fig::different-k}, all algorithms---except \texttt{stage\_1} (as predicted by \Cref{thm::main}) and \texttt{sparse\_filter} (which underperforms at larger sparsity levels $k$)---achieve $\ell_2$-error that remains largely independent of sparsity. In more heavy-tailed settings, depicted in the second row of \Cref{fig::different-k}, all algorithms display an increase in $\ell_2$-error as $k$ grows. Nevertheless, across nearly all scenarios, our full algorithms (\texttt{full\_GD} and \texttt{full\_filter}) outperform the benchmarks. We further hypothesize that the weaker performance of \texttt{full\_filter} for the Pareto distribution with $b=2$ arises from the suboptimal performance of \texttt{sparse\_GD} and \texttt{sparse\_filter} when used in Stage~2.

\begin{figure*}[h]
    \centering
    \includegraphics[width=\textwidth]{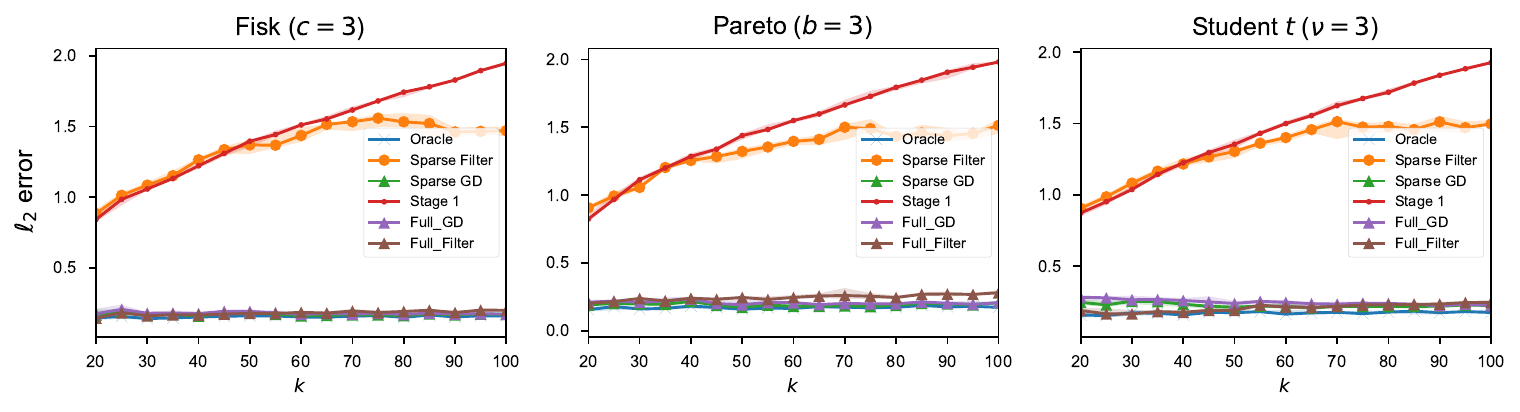}\\
    \includegraphics[width=\textwidth]{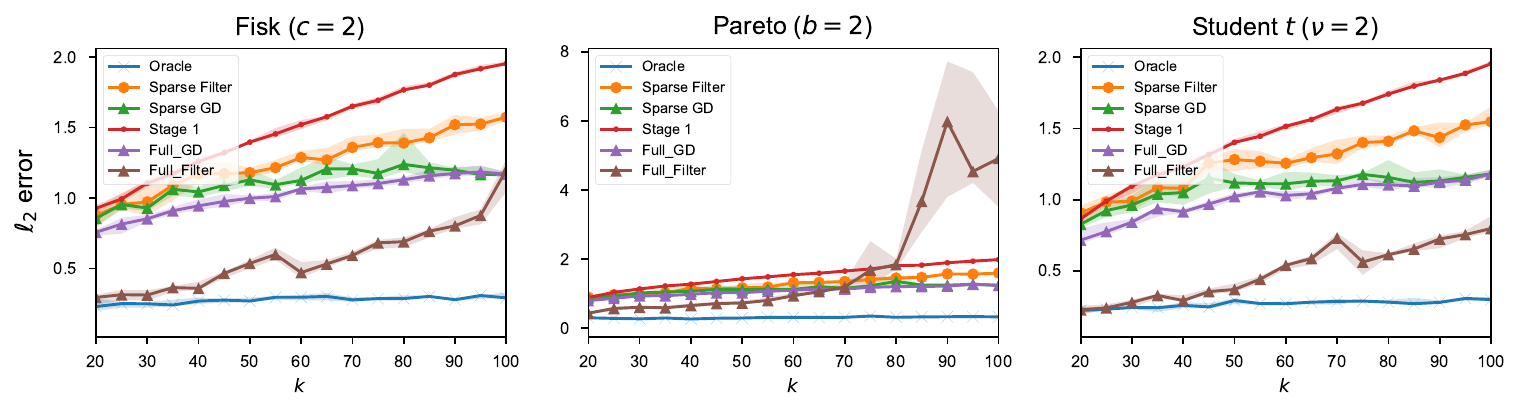}
    \caption{Comparison among different algorithms for varying sparsity levels $k$. The second row corresponds to distributions with infinite variance.}
    \label{fig::different-k}
\end{figure*}
\subsection{Infinite Variance Regime}
In this simulation, we evaluate the performance of the algorithms with respect to the heaviness of the tail distributions. As shown in \Cref{fig:infinite-variance}, we vary the parameters $c, b, \nu$ over the range $[1, 3.5]$. Smaller parameter values correspond to heavier tails, with values in the interval $(1,2]$ resulting in distributions of infinite variance. Our algorithms (\texttt{stage\_1}, \texttt{full\_GD}, and \texttt{full\_filter}) demonstrate superior robustness under these heavy-tailed conditions, highlighting the advantage of our approach.
\begin{figure*}[h]
    \centering
    \includegraphics[width=\textwidth]{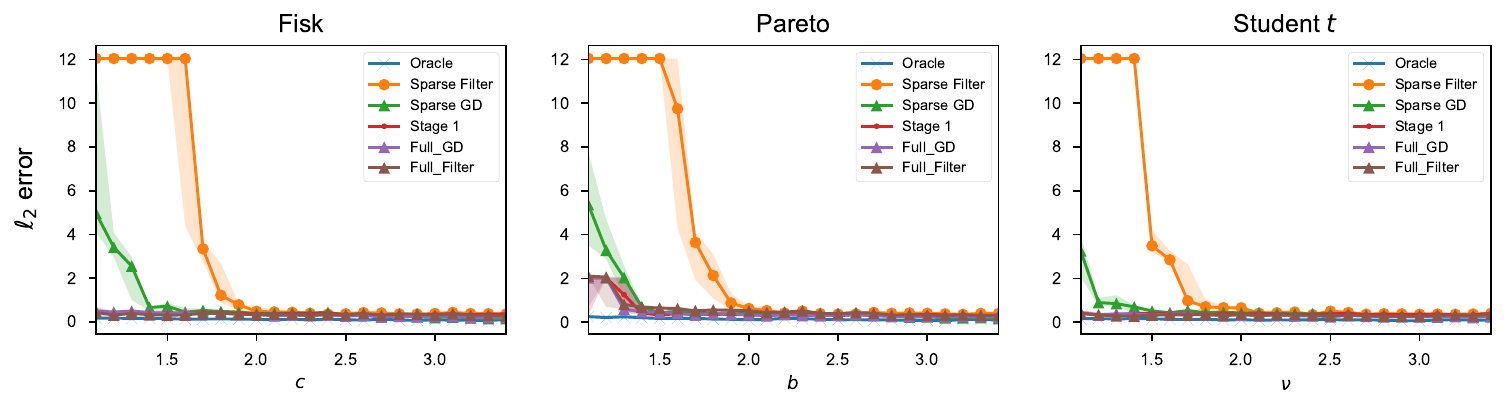}
    \caption{Comparison among different algorithms in the infinite variance regime.}
    \label{fig:infinite-variance}
\end{figure*}

\subsection{Performance with Different $\epsilon$}
In this simulation, we study the relationship between the $\ell_2$-error and the corruption ratio $\epsilon$ across all six estimators. As shown in \Cref{fig:different-err}, apart from the \texttt{Oracle}---whose error remains unaffected by $\epsilon$ (as expected)---our proposed algorithms (either single-stage or full version) consistently outperform the alternatives. While our theoretical analysis predicts an $\ell_2$-error of order $\cO(\sqrt{\epsilon})$, the empirical results reveal an approximately linear dependence on $\epsilon$. We attribute this discrepancy to the non-adversarial nature of the outlier model used in our experiments.
\begin{figure*}[h!]
    \centering
    \includegraphics[width=\textwidth]{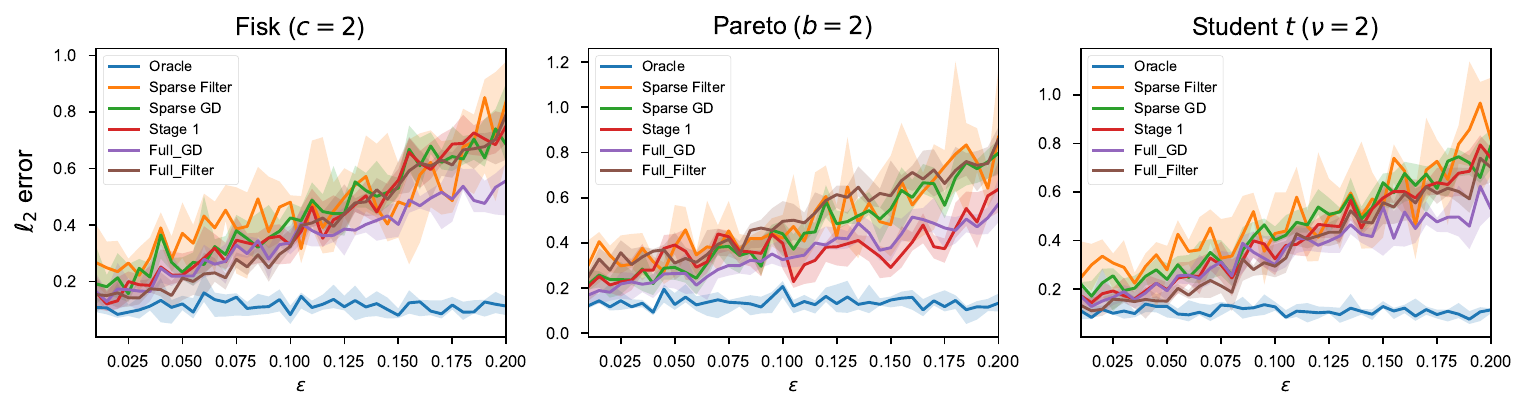}
    \caption{Comparison among different algorithms for different corruption rates $\epsilon$.}
    \label{fig:different-err}
\end{figure*}

\subsection{Running Time}
Next, we examine the running time of our algorithms. Specifically, we run $600$ iterations for \texttt{stage\_1}, while in the full algorithms we restrict Stage~1 to $200$ iterations. As shown in \Cref{fig::runtime}, all algorithms exhibit linear runtime, consistent with our theoretical guarantees.
\begin{figure*}[h]
    \centering
    \includegraphics[width=0.5\textwidth]{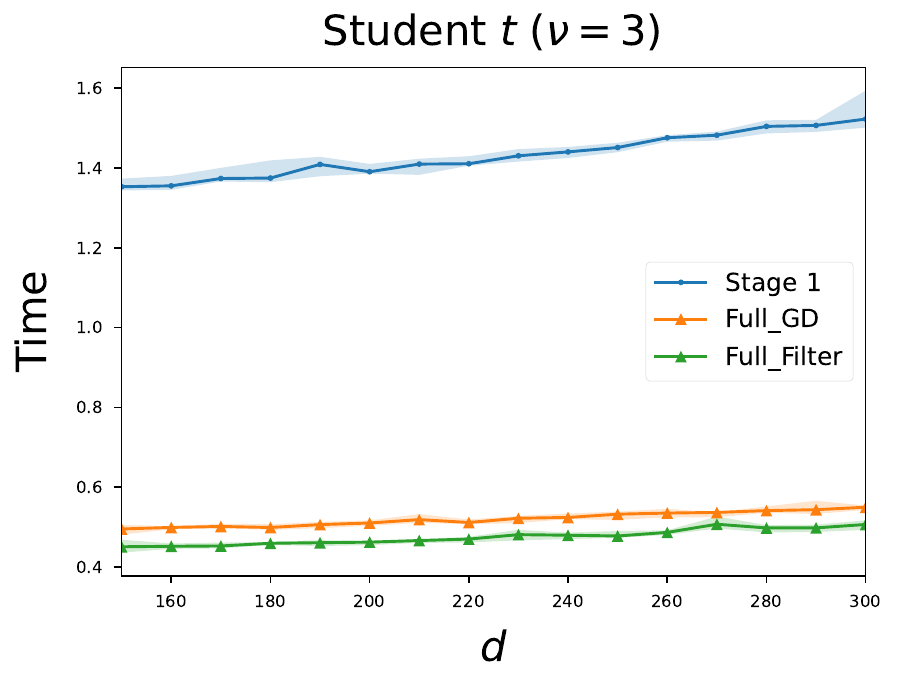}
    \caption{Running time of the proposed single-stage and full algorithms as a function of dimension.}
    \label{fig::runtime}
\end{figure*}